\documentclass{article}
\usepackage{arxiv}

\usepackage[utf8]{inputenc} 
\usepackage[T1]{fontenc}    
\usepackage{hyperref}       
\usepackage{url}            
\usepackage{booktabs}       
\usepackage{amsfonts}       
\usepackage{nicefrac}       
\usepackage{microtype}      
\usepackage{paralist,multirow,graphicx,amsmath,amsfonts,amssymb,epstopdf,color,epsfig,ragged2e,array,standalone,booktabs,subfig,float}
\usepackage{stfloats}
\graphicspath{ {./images/} }

\usepackage{array}
\usepackage{tabularx}
\usepackage{amsmath}
\usepackage{amsthm}
\usepackage{mathtools}

\newcommand{\reals}{\ensuremath{\mathbb{R}}}
\newcommand{\naturals}{\ensuremath{\mathbb{N}}}
\newcommand{\defeq}{\vcentcolon=}
\newcommand{\eqdef}{=\vcentcolon}
\newcommand{\expec}{\ensuremath{\mathbb{E}}}

\newcommand{\htanh}{\ensuremath{\text{htanh}}}

\newcommand{\cN}{\ensuremath{\mathcal{N}}}

\newcommand{\cB}{\ensuremath{\mathcal{B}}}

\newcommand{\cF}{\ensuremath{\mathcal{F}}}

\newcommand{\cO}{\ensuremath{\mathcal{O}}}
\newcommand{\cD}{\ensuremath{\mathcal{D}}}
\newcommand{\cL}{\ensuremath{\mathcal{L}}}

\newcommand{\vz}{\mathbf z}
\newcommand{\vx}{\mathbf x}

\newcommand{\vb}{\mathbf b}
\newcommand{\vh}{\mathbf h}

\newcommand{\mW}{\ensuremath{\mathbf{W}}}

\newtheorem{theorem}{Theorem}[section]
\newtheorem{corollary}{Corollary}[theorem]
\newtheorem{lemma}[theorem]{Lemma}
\newtheorem{definition}{Definition}[section]

\title{Activation function design for deep networks: linearity and effective initialisation}

\author{
 $\dagger \ddag$ M.~Murray, $\dagger$ V.~Abrol, $\dagger \ddag$ J.~Tanner \\
  $\dagger$ Mathematical Institute, University of Oxford, UK\\
  $\ddag$ The Alan Turing Institute, London, UK\\
  \texttt{[murray,abrol,tanner]@maths.ox.ac.uk} \\
   }

\begin{document}
\maketitle
\begin{abstract}
The activation function deployed in a deep neural network has great influence on the performance of the network at initialisation, which in turn has implications for training. In this paper we study how to avoid two problems at initialisation identified in prior works: rapid convergence of pairwise input correlations, and vanishing and exploding gradients. We prove that both these problems can be avoided by choosing an activation function possessing a sufficiently large linear region around the origin, relative to the bias variance $\sigma_b^2$ of the network's random initialisation. We demonstrate empirically that using such activation functions leads to tangible benefits in practice, both in terms test and training accuracy as well as training time. Furthermore, we observe that the shape of the nonlinear activation outside the linear region appears to have a relatively limited impact on training. Finally, our results also allow us to train networks in a new hyperparameter regime, with a much larger bias variance than has previously been possible. 
\end{abstract}

\section{Introduction}\label{sec:intro}
The preactivations $\vh^{(l)} \in \mathbb{R}^{N_l}$ and activations $\vz^{(l)}(\vx) \in \reals^{N_l}$ of an input vector $\vx\in \reals^{N_0}$ at each layer $l\in [L]$ of a fully-connected, deep, feedforward neural network, with weight matrices $\mathbf{W}^{(l)} \in \mathbb{R}^{N_l\times N_{l-1}}$ and bias vectors $\mathbf{b}^{(l)} \in \mathbb{R}^{N_l}$, are computed by the following pair of recurrence relations,
\begin{equation}
    \vz^{(l)}(\vx)=\mathbf{W}^{(l)} \mathbf{h}^{(l-1)} + \mathbf{b}^{(l-1)}, \; \mathbf{h}^{(l)}(\vx)=\phi(\vz^{(l)}(\vx)),
    \label{eq:network}
\end{equation}
where $\phi: \mathbb{R\rightarrow R}$ is the pointwise nonlinearity, referred to as the activation function, deployed at each layer. The impact of the choice of activation function $\phi$ on a wide range of a properties, including for instance the trainability and generality of the network, is an active area of research \cite{glorot10a, searching_activations, parametric_elu, tanhexp, Selu, ELU}. In this paper we study specifically how the choice of activation function impacts the performance of the network at initialisation, and the subsequent implications of this on training. We build on prior work \cite{Poole2016,samuel2017,pennington18a,hayou19a} concerning randomly initialised deep neural networks (DNNs), focusing on two objectives: first, ensuring deeper information propagation by slowing the rate of convergence of pairwise input correlations, and second, avoiding vanishing and exploding gradients \cite{Hochreiter98} by at least approximately achieving dynamical isometry. In \cite{Poole2016} estimates for how a) the length of the preactivations of an input and b) the correlation between the preactivations of two distinct inputs evolve with depth were derived. In both \cite{Poole2016, samuel2017} the authors investigated how the choice of the activation function $\phi$ and the variance of the random weights $\sigma_w^2$ and biases $\sigma_b^2$ effect the lengths of and correlations between preactivations. A key finding of \cite{samuel2017} was the identification of a condition on $(\sigma_w^2,\sigma_b^2)$, termed initialisation on the edge of chaos (EOC), which is necessary in order to avoid asymptotic exponentially fast convergence of both the lengths and correlations. Furthermore, \cite{hayou19a} identified for a broad range of activation function that by letting $\sigma_b^2 \rightarrow 0$ then the rate of convergence in correlation can be reduced away from the fixed point, i.e., non-asymptotically. We review this line of research in detail in Section \ref{subsec:MF_forward}. A related line of work \cite{pennington18a, Saxe13, pennington17a} proposed that the problem of vanishing and exploding gradients can be mitigated by ensuring that the spectrum of the input-output Jacobian of the network is concentrated around one, a property termed dynamical isometry. In addition, it was shown that initialisation on the EOC implies that the mean of the spectrum of the input-output Jacobian is one. Therefore, dynamical isometry can at least approximately be achieved by making, through careful selection of $(\phi,\sigma_w^2,\sigma_b^2)$, the variance $\sigma_{JJ^T}^2$ of the spectrum close to zero. More specifically, the authors highlighted how this can be achieved for particular activation functions by using an orthogonal initialisation scheme for the weights and again by letting $\sigma_b^2 \rightarrow 0$. A number of algorithmic and architectural methods have been proposed and deployed in practice to overcome the problem of vanishing and exploding gradients across a variety of contexts, see for instance \cite{glorot10a,Selu,Hochreiter1997,kaiming16, rupesh15,He2015,Ioffe2015,ba2016}. For brevity we do not review these here, but remark that the line of work described above can be viewed as complimentary to these methods, focusing specifically on the network at initialisation.

Letting $\sigma_b^2 \rightarrow 0$ is unsatisfactory for a number of reasons, first, $\sigma_b^2 \rightarrow 0$ typically results in the length of the preactivations converging towards zero with depth. Second, and as demonstrated by our experiments in Section \ref{sec:experiments}, more affine initialisations with a small but not overly small $\sigma_b^2$ appear to provide the best training outcomes in practice. Third and finally, this strategy does not leverage our freedom in designing the activation function deployed in order to achieve these two goals. In this paper we introduce and study a novel set of scaled, bounded activation functions which are linear in some neighbourhood of the origin. For functions in this set we derive upper bounds controlling both the rate of convergence of correlations between preactivations, as well as the variance of the spectrum of the input-output Jacobian. Notably these bounds depend on the ratio $\sigma_b^2/a^2$, where $a$ denotes the size of the linear region of the activation function, and approach zero as $\sigma_b^2/a^2 \rightarrow 0$. The theory presented in this paper therefore provides a rigorous explanation for the observation that networks equipped with an activation function which approximates the identity at the origin often train well (see for example \cite{hamed17}). Furthermore, our experiments demonstrate the benefits of deploying these scaled, bounded activations in practice, in terms of train and test accuracy, training time and allowing us to train networks with a much larger bias variance than was previously possible. 

The structure of this paper is as follows, in Section \ref{sec:background} we review the foundational results of prior works and summarise the contributions of this paper, in particular Theorem \ref{theorem:main}. In Section \ref{sec:theory} we derive Theorem \ref{theorem:main}, then in Section \ref{sec:experiments} we investigate the implications of our theory for training deep networks in practice.
\section{Principles for initialising deep networks}\label{sec:background}
As referenced to in Section \ref{sec:intro}, the prior works on which our results are based can be divided into two distinct themes, one concerned with the dynamics of the preactivation correlations in the forward pass \cite{Poole2016,samuel2017,hayou19a}, and the other dynamical isometry \cite{pennington18a, Saxe13, pennington17a}. We review these two themes in detail in Sections \ref{subsec:MF_forward} and \ref{subsec:MF_backward} respectively, then in Section \ref{subsec:paper_contributions} we introduce and present the key contribution of this paper, Theorem \ref{theorem:main}. Note that in what follows $Z,Z_1, Z_2 \sim \cN(0,1)$ are considered to be independent and identically distributed standard Gaussian random variables. Furthermore, the standard Gaussian measures in one and two dimensions are be denoted by $\gamma$ and $\gamma^{(2)}$ respectively.

\subsection{Dynamics of the preactivation correlations in the forward pass} \label{subsec:MF_forward}
In \cite{radford_book} it was established that the outputs of certain random, feedforward, single hidden layer neural networks converge in distribution to centred Gaussian processes as the width of the hidden layer goes to infinity. More recently in \cite{matthews2018gaussian, LeeBNSPS18}, this result was extended to random multilayer neural networks, with the preactivations at each layer being found to converge in distribution to centred Gaussian processes in the large width limit. The behaviour of a centred Gaussian process is fully described by its covariance matrix. In particular, for networks whose forward pass is described by \eqref{eq:network}, with the weights and biases at each layer $l \in [L]$ being mutually independent, centred Gaussian random variables with variances $\sigma_w^2/N_{l-1}$ and $\sigma_b^2$ respectively, and assuming that the same activation function $\phi$ is deployed at each neuron, then the kernel used to compute the entries of the covariance matrix of the Gaussian process at the $l$th layer is defined by the following recurrence relation, 
\begin{equation}\label{equation:GP_kernel}
\kappa^{(l)}(\vx^{(\alpha)}, \vx^{(\beta)}) = \sigma_w^2 \expec[\phi (z_i^{(l)}(\vx^{(\alpha)}))\phi(z_i^{(l)}(\vx^{(\beta)}))] + \sigma_b^2.
\end{equation}
Note here that $i \in [N_l]$ can be any entry of the preactivation. This recurrence relation has been highlighted a number of times in a variety of contexts, the one most relevant for this paper being a mean field approximation used to study signal propagation in the forward pass \cite{Poole2016, samuel2017,hayou19a}. In this work we do not focus on the correspondence between wide, random neural networks and Gaussian processes as in \cite{matthews2018gaussian,LeeBNSPS18}. Rather, as in \cite{Poole2016, samuel2017,hayou19a}, we adopt the Gaussian process or mean field approximation in order to better understand and analyse key statistics of the network at initialisation. In particular, and as highlighted in \cite{Poole2016}, in the infinite width limit one can interpret the variance of the preactivations of an input at a given layer, corresponding to \eqref{equation:GP_kernel} with $\alpha \neq \beta$, as the expected euclidean length of said input at said layer. Adopting the notation of \cite{Poole2016}, then the sequence of variances, or expected lengths, $( q_{\alpha}^{(l)})_{l=1}^L$, associated with an arbitrary input $\vx^{(\alpha)}$, are generated by the following recurrence relation,
\begin{equation} \label{eq:var_map}
\begin{aligned}
    & q_{\alpha}^{(l)} = V_{\phi}(q_{\alpha}^{(l-1)}):= \sigma_w^2\int_{\reals}\phi \left(\sqrt{q_{\alpha}^{(l-1)}}z \right)^2 d\gamma(z)+\sigma_b^2 
    =  \sigma_w^2 \expec \left[\phi \left(\sqrt{q_{\alpha}^{(l-1)}}Z \right)^2\right] + \sigma_b^2
\end{aligned}
\end{equation}
for $l \in [L]L$ with $q_{\alpha}^{(0)} \defeq ||\vx^{(\alpha)} ||_2^2$. As demonstrated in Figure \ref{fig:var_corr_funcs_common}(a), analysis of the variance function $V_{\phi}: \reals_{\geq 0} \rightarrow \reals_{\geq 0}$ provides insight into the impact of $(\phi,\sigma_w^2,\sigma_b^2)$ on the dynamics of the expected length of the preactivations with depth. Indeed, unless $(\phi,\sigma_w^2,\sigma_b^2)$ is chosen appropriately then $q_{\alpha}^{(l)}$ can either rapidly converge towards zero or diverge. This problem of vanishing or exploding activation lengths can be mitigated by choosing $(\phi,\sigma_w^2,\sigma_b^2)$ in order that there exists a stable fixed point $q^* = V_{\phi}(q^*)$ for which $q_{\alpha}^{(l)} \rightarrow q^*$ for all inputs $\vx^{(\alpha)}$, see for example Tanh with $\sigma_b^2=0.1$ in Figure \ref{fig:var_corr_funcs_common}(a). We remark that this problem was also studied in \cite{boris19} in the context of finite width deep ReLU networks.

\begin{figure}[ht]
	\centering
	\subfloat[]{\includegraphics[scale=.45]{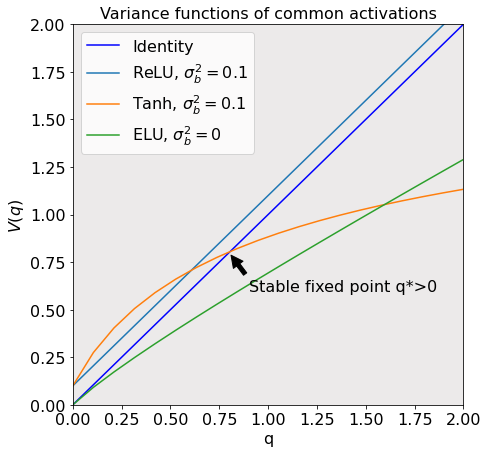}}
	\subfloat[]{\includegraphics[scale=.45]{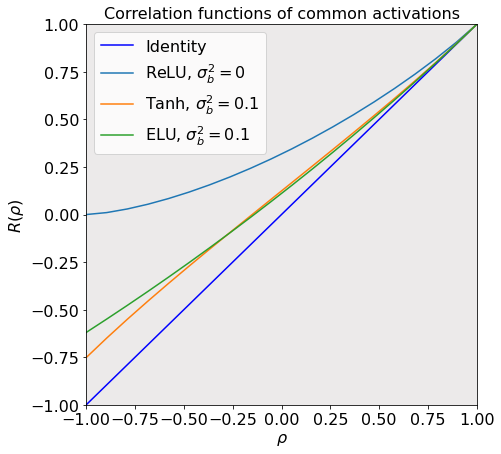}}
	\caption{\small 
	(a) Variance functions for ReLU, Tanh and ELU with $\sigma_b^2 \in \{0, 0.1\}$. With $\sigma_b^2 =0.1$ the variance function associated with ReLU has no fixed point and $q^{(l)}\rightarrow \infty$ from any initial variance $q^{(0)} \in \reals_{\geq 0}$. With $\sigma_b^2 =0$ the variance function associated with ELU has a unique and marginally stable point $q^* =0$, therefore $q^{(l)}\rightarrow 0$ from any initial variance $q^{(0)} \in \reals_{\geq 0}$. Finally, with $\sigma_b^2 =0.1$, the variance function associated with Tanh has a unique and stable fixed point $q^*>0$, therefore $q^{(l)}\rightarrow q^*$ from any initial variance $q^{(0)} \in \reals_{\geq 0}$. (b) Correlation functions associated with Relu, Tanh and ELU with $\sigma_b^2 \in \{0, 0.1\}$. For ReLU, and as illustrated by (a), for there to exist a fixed point of the variance function it is necessary that $\sigma_b^2=0$. Note that without a fixed point $q^*$ it is not possible to define a correlation function which is fixed with depth. Relative to the other activation functions considered, the correlation function of ReLU is further from the identity function, and as a result correlations converge faster with the depth. This can also be observed in Figure~\ref{fig:conv_of_corrs} in Appendix~\ref{appendix:conv_of_corrs}. Note that in both (a) and (b) $\sigma_w^2$ is chosen so that $\chi_1 =1$, see \eqref{eq:chi}.
	}
	\label{fig:var_corr_funcs_common}
\end{figure}

In \cite{Poole2016} the covariance and correlation between preactivations of different inputs was also analysed and a recurrence relation for the covariance derived, which is equivalent to \eqref{equation:GP_kernel} with $\alpha \neq \beta$. In the infinite width limit the covariance can be interpreted as the inner product between the preactivations of two inputs, and the correlation as the angle. The evolution of these quantities with depth has important implications for the performance of the network at initialisation as well as subsequent training. Denoting the covariance between the preactivations of $\vx^{(\alpha)}$ and $\vx^{(\beta)}$ at the $l$th layer as $q_{\alpha \beta}^{(l)}$, with respective variances $q_{\alpha}^{(l)}$ and $q_{\beta}^{(l)}$, then it was shown in \cite{Poole2016} that
\begin{equation} \label{eq:covariance_function}
\begin{aligned}
    & q_{\alpha \beta}^{(l)}=  \sigma_w^2 \int_{\reals^2} \phi(u_1)\phi(u_2)d \gamma^{(2)}(z_1, z_2) + \sigma_b^2, \\ 
    & u_1 := \sqrt{q_{\alpha}^{(l-1)}} z_1,\; u_2 := \sqrt{q_{\beta}^{(l-1)}}(\rho_{\alpha \beta}^{(l-1)}z_1 + \sqrt{1-(\rho_{\alpha \beta}^{(l-1)})^2} z_2),
\end{aligned}
\end{equation}
where $\rho_{\alpha \beta}^{(l-1)} = q_{\alpha \beta}^{(l-1)}/(q_{\alpha}^{(l-1)}q_{\beta}^{(l-1)})$ is the correlation at the previous layer. If the variance function defined in \eqref{eq:var_map} has a fixed point $q^*>0$ then an additional key benefit beyond that already discussed is that the analysis of the evolution of the covariance and correlation with depth can be simplified. In particular, assuming that the inputs are normalised so that $q_{\alpha}^{(0)} = q_{\beta}^{(0)} = q^*>0$, then $q_{\alpha}^{(l)} = q_{\beta}^{(l)} = q^*$ for all $l \in [L]$. In this setting the sequence of correlations $(\rho_{\alpha \beta}^{(l)})_{l=1}^L$ are generated by the following recurrence relation,
\begin{equation} \label{eq:corr_map}
\begin{aligned}
    & \rho_{\alpha \beta}^{(l)} = R_{\phi,q^*}(\rho_{\alpha\beta}^{(l-1)}):=
     \frac{q_{\alpha \beta}^{(l)}}{q^*}=  \frac{\sigma_w^2}{q^*} \int_{\reals}\phi(u_1)\phi(u_2)d \gamma^{(2)}(z_1, z_2) + \frac{\sigma_b^2}{q^*} = \frac{\sigma_w^2}{q^*}\expec[\phi(U_1) \phi(U_2)] + \frac{\sigma_b^2}{q^*}.
\end{aligned}
\end{equation}
Here we refer to $R_{\phi, q^*}: [-1,1] \rightarrow [-1,1]$ as the correlation function, and $U_1 := \sqrt{q^*} Z_1$, $U_2 := \sqrt{q^*}(\rho Z_1 + \sqrt{1- \rho^2}Z_2)$, where $\rho$ is the input argument  $R_{\phi, q^*}(\rho)$, are dependent Gaussian random variables with $U_1, U_2  \sim \cN(0, q^*)$.

Analysing the univariate recurrence relation $\rho^{(l)} = R_{\phi, q^*}(\rho^{(l-1)})$ allows for the identification of both depth limits, beyond which information cannot propagate, as well as issues around stability, in a manner analogous to that of the dynamical systems perspective given in \cite{Haber_2017}. In order to appreciate these points, observe, as in \cite{Poole2016}, that by construction as long as the correlation function is well defined on $[-1,1]$ then one is always a fixed point, 
\[
\begin{aligned}
R_{\phi,q^*}(1) = \frac{\sigma_w^2}{q^*}\expec[\phi(U_1)^2] + \frac{\sigma_b^2}{q^*}
 = \frac{V(q^*)}{q^*}
=1.
\end{aligned}
\]
The general shape of the correlation function, as well as the stability of the fixed point at one and the existence of other fixed points, is determined by the choice of $(\phi,\sigma_w^2,\sigma_b^2)$. This can be observed in Figure \ref{fig:var_corr_funcs_common}(b), in which the correlation functions of three commonly used activation functions are plotted side by side for comparison. Figure \ref{fig:eoc_curves}(a) shows the correlation function for three different values of $\sigma_w^2$ with $\sigma_b^2=0.1$ and $\phi(z) = \text{htanh}(z)$ fixed. For $\sigma_w^2=0.63$ and $\sigma_w^2=1.26$ it is clear from Figure \ref{fig:eoc_curves}(a) that $\rho^*=1$ is the only positive fixed point. Furthermore, for these two choices of $\sigma_w^2$ it is also clear, by inspection, that $\rho^* = 1$ is a stable and marginally stable fixed point respectively, and that from any initial correlation the sequence of correlations $\rho^{(l)} \rightarrow 1$ with depth. As a result, in the infinite width limit any pair of inputs, no matter how large the angle between them is, are mapped to the same point by the network asymptotically with depth. However, for $\sigma_w^2 =0.63$ a new, stable fixed point $\rho^*<1$ is introduced. In this case, then for any initial correlation the sequence $\rho^{(l)}$ converges to the limit $\rho^*<1$. Therefore, in the infinite width limit any pair of inputs, no matter how small the initial angle between them is, are mapped to different points by the network asymptotically with depth. Figure \ref{fig:eoc_curves}(a) therefore illustrates the existence of two very distinct regimes for a given choice of activation function, one ordered, in which all inputs are mapped asymptotically to the same output, and the other chaotic, in which even arbitrarily small perturbations of an input lead to a different output. Figure \ref{fig:eoc_curves}(b) illustrates the limiting $\rho^{*}$ throughout the $(\sigma_w,\sigma_b)$ plane for $\phi=\htanh(\cdot)$.

\begin{figure}[ht]
	\centering
	\subfloat[]{\includegraphics[scale=.45]{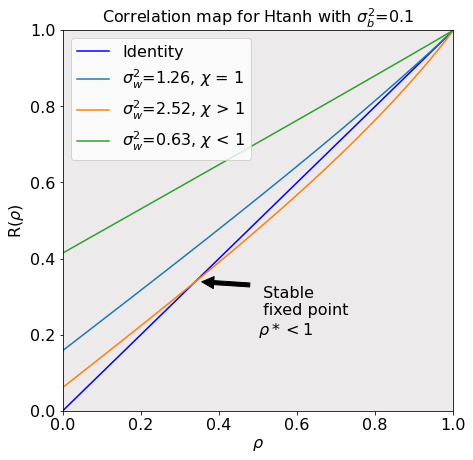}}
	\subfloat[]{\includegraphics[scale=.45]{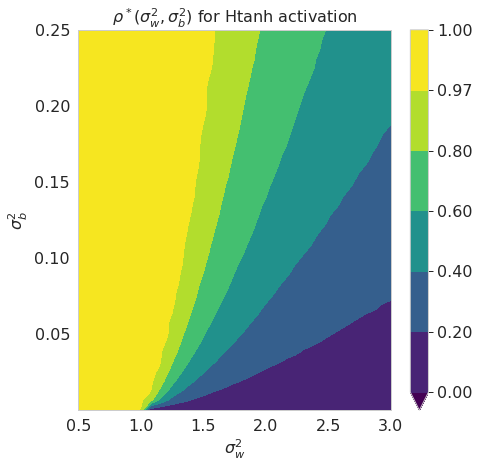}} \\
	\subfloat[]{\includegraphics[scale=.45]{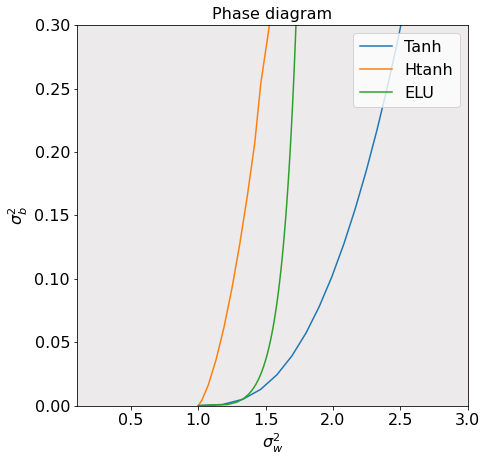}}
	\subfloat[]{\includegraphics[scale=.45]{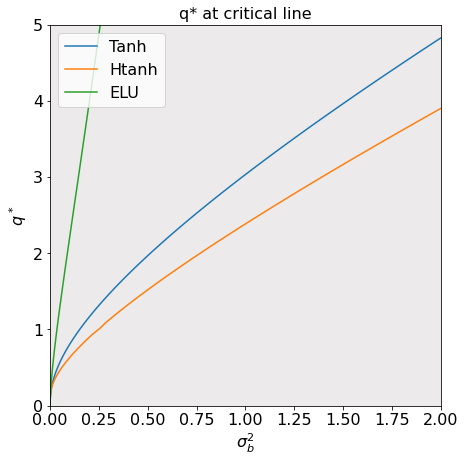}}
	\caption{\small 
	(a) Correlation map for htanh activation function, plotted on the interval $[0,1]$ to better display fixed point behaviour (there are no fixed points in the interval $[-1,0]$). When $\chi \leq 1$ then one is the unique fixed point of the correlation function. Furthermore, for $\chi < 1$ then one is a stable fixed point, and for $\chi = 1$ then one is a marginally stable fixed point. When $\chi > 1$ then although one is still a fixed point it is unstable, and a new stable fixed point less than one is introduced.
	(b) Stable fixed point $\rho^*$ as a function of $(\sigma_w^2,\sigma_b^2)$ for the tanh activation function.
	(c) The EOC critical line $\chi=1$ for different activation functions. 
	(d) $\sigma_w^2$ as a function of $q^*$, as $\sigma_b^2 \rightarrow 0$ then $q^*\rightarrow 0$.
	}
	\label{fig:eoc_curves}
\end{figure}

In order to understand and separate these two regions \cite{Poole2016,samuel2017} studied the slope of $R_{\phi,q^*}(\rho)$ at $\rho=1$: for a sufficiently smooth activation function $\phi$ the derivative is given by
\begin{equation} \label{eq:corr_diff_map}
    R_{\phi,q^*}'(\rho) = \sigma_w^2\expec[\phi'(U_1) \phi'(U_2)],
\end{equation}
the slope at $\rho=1$ is therefore
\begin{equation} \label{eq:chi}
    \chi_1:=R_{\phi,q^*}'(1) = \sigma_w^2\expec[\phi'(U_1)^2].
\end{equation}
The set of pairs $(\sigma_w^2, \sigma_b^2) \in \reals_{\geq0} \times \reals_{\geq 0}$ such that $\chi_1 = 1$ is colloquially referred to in the literature as the edge of chaos (EOC) and in Figure \ref{fig:eoc_curves}(c)  this set, or curve, is plotted for three different activation functions. In \cite{samuel2017} it was proved that if $(\sigma_w^2,\sigma_b^2)$ does not lie on the EOC then the sequences generated by the variance and correlation functions, \eqref{eq:var_map} and \eqref{eq:corr_map} respectively, approach their respective fixed points asymptotically exponentially fast. As a result, the EOC curve can be interpreted as a transition boundary between order, corresponding to $\chi_1 < 1$, in which pairwise correlations converge asymptotically exponentially fast to $1$ with depth, and chaos, corresponding to $\chi_1 > 1$, in which even pairwise input correlations that are arbitrarily close to $1$ diverge asymptotically exponentially fast with depth. The dynamics of the preactivation correlations can also be understood and interpreted in terms of a network's sensitivity to perturbations of the input, with networks initialised in the chaotic regime carrying the risk of being overly sensitive, and those initialised in the ordered regime being insensitive. These observations are also of practical relevance for training, indeed experimental evidence \cite{samuel2017} illustrates that initialising closer to the EOC consistently results in reduced training times.

To recap the discussion presented so far, rapid convergence in correlations to a fixed point appears to result in poor training outcomes, with the network being either highly sensitive or insensitive to perturbations of the input. To avoid asymptotically exponentially fast convergence it is necessary to choose $(\sigma_w^2, \sigma_b^2)$ so that $\chi_1 = 1$. The specific asymptotic rate however depends on the activation function deployed. In \cite{hayou19a}, and under the assumption that initialisation is on the EOC, the asymptotic convergence rate of the sequence of correlations for different families of activation functions was analysed. A key finding of this work is that while ReLU like activation functions have asymptotic convergence $|1-\rho^{(l)}|=\cO(l^{-2})$, a broad class of smooth activation functions, including, tanh, elu and swish, have convergence $|1-\rho^{(l)}|=\cO(l^{-1})$. However, it seems highly reasonable that avoiding fast non-asymptotic, i.e., away for $\rho^*=1$, convergence of correlations is also potentially of value. Figure \ref{fig:var_corr_funcs_common}(b) illustrates, assuming initialisation on the EOC, that certain combinations of $(\phi, \sigma_b^2)$ can result in a correlation function which is closer to the identity function. This in turn implies a slower convergence of the correlation sequence throughout the layers, rather than just at the depth limit. This idea was promoted in \cite[Appendix B.2]{hayou19a}, where it was suggested that one should choose a small $\sigma_b^2$ and an activation function which is approximately linear, for example, one which is the weighted sum of a linear and nonlinear activation function. 


\subsection{Avoiding vanishing and exploding gradients via dynamical isometry} \label{subsec:MF_backward}
A key ingredient behind the success of neural networks is that their modular, compositional structure enables easy computation of the gradients of the network output with respect to the model parameters. The computation of these gradients can efficiently be performed using the backpropagation algorithm \cite{backprop}. Backpropagation, also referred to as the backward pass of the network, generates a sequence of vectors $(\delta^{(l)}(\vx))_{l=1}^{L}$ at each layer with respect to an input $\vx \in \reals^N$, which measure the error associated with minimising the loss $\cL(\theta,x)$. Here $\theta$ denotes the trainable network parameters, $\theta \defeq \bigcup_{l=1}^L \{ \mW^{(l)}\} \cup \{ \vb ^{(l)}\}$.  For a suitably smooth loss function $\cL(\theta,x)$, the backpropagation error vector $\delta^{(l)}(\vx)$, and the gradient of $\cL(\theta,x)$ with respect to biases $b_j^{(l)}$ and weights $w^{(l)}_{j,i}$, at each layer $l \in [L]$ can be computed using the following recurrence equations,
\begin{equation}\label{eq:backprop}
\begin{aligned}
\delta^{(L)}(\vx) &= \mathbf{D}^{(L)}(\vx)\nabla_{\vh^{(L)}} \cL(\theta, \vx),\\
\delta^{(l)}(\vx) &= \left(\mathbf{D}^{(l)}(\vx) \mW^{(l+1)} \right)^T \delta^{(l+1)}(\vx),\\
\frac{\partial \cL(\theta, \vx)}{\partial b_{j}^{(l)}} &= \delta_j^{(l)}(\vx),\\
\frac{\partial \cL(\theta, \vx)}{\partial w_{j,i}^{(l)}} &= h_i^{(l-1)}(\vx)\delta_j^{(l)}(\vx).
\end{aligned}
\end{equation}
Here, for each $l \in [L]$, $\mathbf{D}^{(l)}(\vx)$ is a diagonal matrix with $D_{ii}^{(l)}(\vx)\defeq \phi^{'}(z_i^{(l)}(\vx))$ for all $i \in [N_l]$.  Consequently, the weights and biases of the network can be iteratively updated via gradient descent type algorithms, or some variant thereof, in order that the forward pass improves its performance, as measured typically by averaging $\cL(\theta,\vx)$ over batches of the training data, at executing the task in hand.  We emphasise that in practice, instead of updating the parameters with respect to a single data point, gradients are calculated and averaged over batches of the training data. It is important therefore not to confuse the input-output Jacobian of the network, which we will discuss presently, with the Jacobian of the loss surface over the training data, or batches of the training data.

We now turn our attention to reviewing the findings of the literature concerning the propagation of $\delta^{(l)}$ in the backwards pass. Analogous to the recurrence relation for the variance function \eqref{eq:var_map} in the forward pass, in \cite{samuel2017} a mean field approximation for $\delta^{(l)}$ was developed. Specifically, the recurrence relation for the variance, $\tilde{q}^l\defeq \mathbb{E}[(\delta_i^l)^2]$, of entries of the error vectors $\delta_i^l$ was shown to be
\begin{equation} \label{eq:var_error_signals}
    \tilde{q}^l=\frac{N^{l+1}}{N^l} \tilde{q}^{l+1} \chi_1
\end{equation}
with $\chi_1$ defined as in \eqref{eq:chi}.  Although certain assumptions needed for this approximation are undesirable, notably the weights used during forward propagation are drawn independently from the weights used in backpropagation, empirical evidence indicates that it is nonetheless a useful model \cite{samuel2017}. In particular, unlike the analysis of the gradients used to update the network parameters. The reappearance of $\chi_1$ in \eqref{eq:var_error_signals} shows the benefit of selecting $(\phi,\sigma_w^2,\sigma_b^2)$ such that  $\chi=1$, as this ensures, at least in expectation, that the magnitudes of the error vectors are stable during the backward pass, i.e., they neither converge to zero or diverge.

Unfortunately, initialising on the EOC suffices only to ensure that vanishing and exploding gradients are avoided on average, and in practice does not guarantee good training performance. In particular, if, as it propagates backwards, an error vector is projected onto a smaller and smaller subspace then the directions in which the parameters in the lower layers can be updated becomes limited, thereby resulting in model degeneracy \cite{Saxe13,pennington17a}. In order to improve training, in \cite{Saxe13,pennington17a} the authors proposed that the product of Jacobians at each layer should act as an isometry on as large a subspace as possible, implying that the singular values of the input-output Jacobian at any layer should concentrate around one. As in \cite{pennington18a,pennington17a}, we denote the input-output Jacobian of the network as $\mathbf{J}: \reals^{N_0} \rightarrow \in\mathbb{R}^{N_0\times N_L}$,
\begin{equation} \label{eq:jacobian}
    \mathbf{J}(\vx) = \prod_{l=1}^L \mathbf{D}^{(l)}(\vx)\mathbf{W}^{(l)}.
\end{equation}
Comparing \eqref{eq:jacobian} with \eqref{eq:backprop}, then the Jacobian of $\vx$ can be viewed as the transpose of the product of linear backpropagation operators used to compute the error vectors associated with $\vx$ at each layer. In \cite{pennington18a,pennington17a} the authors considered two initialisation schemes, in both the biases $\mathbf{b}^{(l)}_i$ at each layer are mutually independent and identically distributed Gaussians with mean $0$ and variance $\sigma_b^2$. The weight matrices at each layer however either have mutually independent identically distributed Gaussian entries with mean 0 and variance $\sigma_w^2/N_{l-1}$, or are drawn from a uniform distribution over scaled orthogonal matrices such that $(\mathbf{W}^{(l)})^T\mathbf{W}^{(l)} = \sigma_w^2 \mathbf{I}$. To analyse the Jacobian of the network at an arbitrary point $\vx$, the authors considered again the large width setting, in which, as per the discussion in Section \ref{subsec:MF_forward}, the preactivations at layer $l$ layer can be modelled as mutually independent, identically distributed, centred Gaussian random variables with variance $q^{(l)}$ \footnote{Note that \eqref{eq:var_map} was derived in the context of Gaussian initialisation, for now we also assume it holds true also for the Orthogonal case.}. For both of the initialisation schemes described, the limiting spectral density of $\mathbf{J}(\vx)$, in terms of its moment generating transform
\begin{equation}
     M_{JJ^T}(z)=\sum_{k=1}^{\infty} \frac{m_k}{z^k},
    \label{eq6}
\end{equation}
was computed and analysed using tools from free probability \cite{pennington18a, pennington17a,pennington17}. We refer the reader to Appendix 6 of \cite{pennington18a} for further details. In the context of ensuring dynamical isometry, then of particular interest are the first and second moments, which can be expressed in terms of the moments $\mu_k$ of $\mathbf{D}^2$ and $s_k$ of $(\mathbf{W}^T\mathbf{W})$. Here we drop the dependence on the layer index $l$ by assuming that the condition $q^{(0)}=q^*$ holds true for all inputs $\vx$. Adopting again the notation used in \cite{pennington18a},
\begin{equation} \label{eq:moments}
\begin{aligned}
     \mu_k &:= \int \phi'(\sqrt{q^*z})^{2k} d \gamma(z)= \expec[ \phi'(\sqrt{q^*})Z)^{2k}],\\
     m_1 &:=(\sigma_w^2\mu_1)^L=(\chi_1)^L,\\
     m_2 &:=(\chi_1)^{2L}L\Big(\frac{\mu_1}{\mu_2} +\frac{1}{L} -1-s_1\Big).
\end{aligned} 
\end{equation}
It is evident from  \eqref{eq:moments} that the mean squared singular value $m_1$ of the Jacobian either exponentially explodes or vanishes unless the network is initialised on the EOC. However, although initialisation on the EOC makes $m_1$ independent of depth, the variance \begin{equation} \label{eq:jacob_var}
     \sigma_{{JJ}^T}^2=m_2-m_1^2=L\Big(\frac{\mu_2}{\mu_1^2} -1-s_1\Big)
\end{equation}
grows linearly with depth. As a result, under the limiting width assumption then for an arbitrary, normalised input vector initialisation on the EOC may still result in an increasingly ill-conditioned Jacobian with depth $L$.  Indeed, we require that $m_1=1$ and $\sigma_{{JJ}^T}^2\approx 0$ so as to at least approximately achieve dynamical isometry. Equation (\ref{eq:jacob_var}) shows that using a Gaussian initialisation scheme, corresponding to $s_1 = -1$, results in a linear growth in $\sigma_{JJ^T}^2$ with depth regardless of the activation function used. As a result, deep, dense feed-forward networks cannot achieve dynamical isometry using Gaussian initialisation. In contrast, orthogonal initialisation, corresponding to $s_1 = 0$, can in theory be used to ensure that $\sigma_{JJ^T}^2$ is arbitrarily close to one by controlling the moment ratio $\mu_2 / \mu_1^2$. However, the ability to control the moment ratio depends on the activation function deployed: in the case of ReLU for instance, whose EOC is a singleton, the moment ratio is a constant and hence once again dynamical isometry cannot be achieved. As highlighted in \cite{pennington18a}, for certain activation functions, e.g., erf and tanh, which have a non-singleton EOC as illustrated by Figure \ref{fig:eoc_curves}(c), then the moment ratio can be reduced by shrinking $q^*$, which in turn can be achieved by reducing $\sigma_b^2$, see Figure \ref{fig:eoc_curves}
(d). Finally we remark that similar analyses were conducted in \cite{NEURIPS2018_13f9896d} in the context of finite width deep ReLU networks.

\subsection{Contribution: activations which approximately preserve correlations and achieve dynamical isometry} \label{subsec:paper_contributions} 
To recap, in Section \ref{subsec:MF_forward} we discussed how, in order to achieve deep information propagation and avoid the network being either highly sensitive or insensitive to perturbations of the input, it is important to avoid a rapid rate of convergence of the correlation. To this end it is necessary to choose $(\phi, \sigma_w^2, \sigma_b^2)$ so that $\chi_1 =1$ and the associated correlation function is close to identity. In Section \ref{subsec:MF_backward} we discussed the problem of vanishing and exploding gradients. Prior works suggest that this problem can be avoided if the singular values of the input-output Jacobian at each layer concentrate around 1. The condition $\chi_1=1$ is equivalent to ensuring that the mean of the input-output Jacobian's spectrum at any layer is one. Therefore, if the variance $\sigma_{JJ^T}^2$ of the spectrum, defined in (\ref{eq:jacob_var}), is close to zero, then the spectrum of the input-output Jacobian is guaranteed to concentrate around one. Furthermore, with orthogonal initialisation on the EOC then so long as $\mu_2/\mu_1^2 \approx 1$ then $\sigma_{JJ^T}^2 \approx 0$.

In this paper we present principles for choosing the activation function $\phi$, see Definition \ref{def:Omega}, and an additional linear scaling parameter which simultaneously enables uniform convergence of the correlation function $R_{\phi,q^*}(\rho)$ to the identity map, and convergence of the variance $\sigma_{JJ^T}^2$ of the input-output Jacobian's spectrum to zero, without requiring $\sigma_b^2$ to shrink towards zero. For further details we refer the reader to Theorem \ref{theorem:main}. From prior work it was unclear how one could achieve this. For example, initialising on the EOC, using orthogonal initialisation and deploying the modulus activation function achieves perfect dynamical isometry, but results in rapid convergence of correlations in the forward pass, leading to poor training outcomes. Treated separately, a key theme that emerges in the prior works towards achieving both goals is that of reducing $\sigma_b^2$ in order to make $q^*$ small. As mentioned briefly in section \ref{subsec:MF_forward}, in \cite{hayou19a} the authors highlight that, for a smooth class of activation functions, taking this action can reduce the gap between the correlation map and the identity. The authors also provide a rule for selecting $\sigma_b^2$ just small enough so that pairwise correlations avoid being within some $\epsilon>0$ of one at the output layer. Furthermore, experiments conducted in \cite{hayou19a} indicate that selecting the variance hyperparameters in this manner accelerates training in practice. However, the authors also highlight that as the depth of the network increases this rule still requires $\sigma_b^2 \rightarrow 0$. Likewise, and as mentioned in section \ref{subsec:MF_backward}, in \cite{pennington18a} the authors show, for the htanh and erf activation functions, that reducing $q^*$ by shrinking $\sigma_b^2$ results in the moment ratio converging towards one. However, this solution is not entirely satisfactory, first if the expected length of the activations goes to zero then asymptotically the network will on average annihilate inputs and fail to pass information to the output. Second, experimental evidence indicates that a small, but not overly small $\sigma_b^2$, gives the best results in practice, which we demonstrate in Section \ref{sec:experiments}. Third, this solution relies only on the choice of $\sigma_b^2$ and ignores the role and design of $\phi$. Our work seeks to address this issue by achieving both goals through the design of $\phi$, even when $\sigma_b^2$ and $q^*$ are fixed away from zero.

In light of the above, and assuming $\sigma_b^2$ is fixed and that $\sigma_w^2$ is chosen in order that $\chi_1 = 1$,  or goal is to design or choose $\phi$ in order that both $\max_{\rho \in [-1,1]}|R_{\phi, q^*}(\rho)-\rho|\approx 0$ and $|\mu_2/\mu_1^2-1| \approx 0$. As highlighted in \cite{hayou19a}, linear activation functions are advantageous from the perspective of slowing the convergence of the pairwise correlations in the forward pass. Likewise, \cite{pennington17, Saxe13} highlight that linear networks with orthogonal initialisation achieve perfect dynamical isometry. However, linear activations are not a suitable option as a linear network can only represent linear functions. To preserve the rich approximation capabilities of nonlinear networks we analyse a set of activation functions characterised by being odd, bounded, Lipschitz continuous and linear around the origin. We refer to activation functions of this type, defined in Definition \ref{def:Omega}, as \textit{scaled-bounded activations}.
\begin{definition}[\textbf{scaled-bounded activations}]\label{def:Omega}
    We refer to the set of activation functions $\phi: \reals \rightarrow \reals$ which satisfy the following properties as scaled-bounded activations.
    \begin{enumerate}
        \item Continuous.
        \item Odd, meaning that $\phi(z) = -\phi(-z)$ for all $z \in \reals$.
        \item Linear around the origin and bounded: in particular there exists $a,k \in \reals_{>0}$ such that $\phi(z) = kz$ for all $z \in [-a,a]$ and  $\phi(z) \leq ak$ for all $z \in \reals$.
        \item Twice differentiable at all points $z \in \reals \backslash \cD$, where $\cD \subset \reals$ is a finite set. Furthermore $|\phi'(z)| \leq k$ for all $z \in \reals \backslash \cD$.
    \end{enumerate}
\end{definition}    
As illustrated in Figure \ref{fig:omega_examples}, the structure of scaled-bounded activation functions outside of their linear region $[-a,a]$ can vary substantially. We note that for any scaled-bounded activation $\phi$ there exists $a,k \in \reals_{>0}$ such that $|\phi(z)| \leq |\text{Shtanh}_{a,k}(z)|$, where
\begin{equation} \label{eq:shtanh}
\text{Shtanh}_{a,k}(z)= 
\begin{cases}
zk, \;\; |z|<a\\
ak, \;\; |z|\geq a.
\end{cases}
\end{equation}
Here Shtanh stands for scaled-bounded hard tanh, and we further note that $\text{Shtanh}_{a,k}(\cdot)$ is a scaled-bounded activation for any $a,k \in \reals_{>0}$. Our main contribution is Theorem \ref{theorem:main}, which asserts, for scaled-bounded activations, that the correlation function $R_{\phi,q^*}$ can be made arbitrarily close to the identity on $[0,1]$, and the moment ratio $\mu_2/\mu_1^2$ arbitrarily close to one, by choosing $\sigma_b^2/a^2$ to be small. 

\begin{figure}[!ht]
   \begin{minipage}{\textwidth}
     \centering
     \includegraphics[scale=0.45]{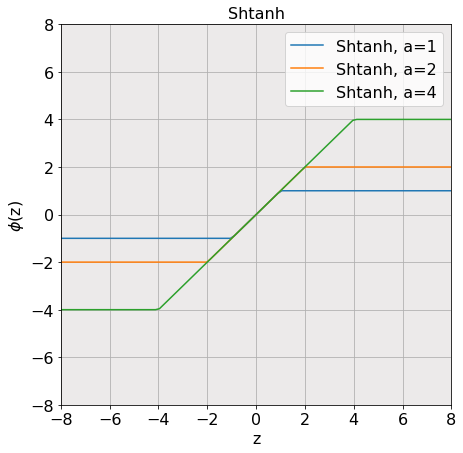}
     \includegraphics[scale=0.45]{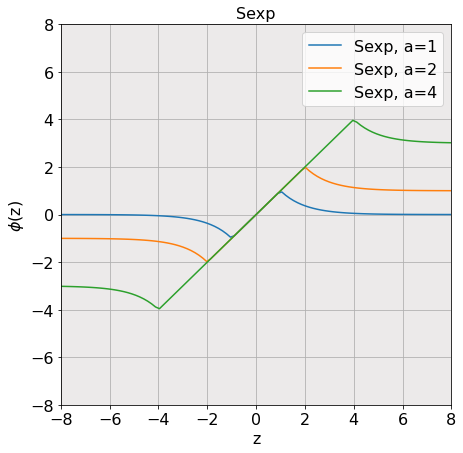}
     \includegraphics[scale=0.45]{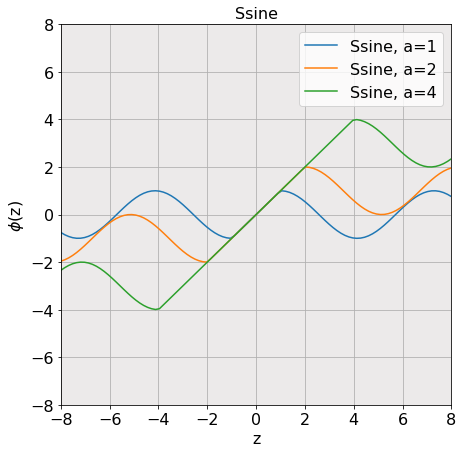}
     \includegraphics[scale=0.45]{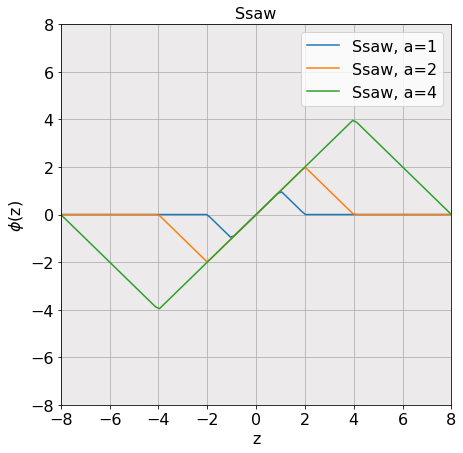}
   \end{minipage}\\[1em]
   \caption{examples of scaled-bounded activation functions with $k=1$ and $a \in \{1,2,4\}$. The prefix S refers to the scaling of the linear region via the parameter $a$.}\label{fig:omega_examples}
\end{figure}

\begin{theorem} \label{theorem:main}
Let $\phi$ be a scaled-bounded activation, see Definition \ref{def:Omega}, $\sigma_b^2>0$ and suppose that
    \[
        \chi_1 \defeq \sigma_w^2\expec[\phi'(\sqrt{q^*}Z)^2] = 1,
    \]
where $q^*>0$ is a fixed point of the associated variance function $V_{\phi}$. In addition, assume that all inputs $\vx$ are normalised so that $||\vx||_2^2 = q^*$. With $y \defeq \frac{\sigma_b^2}{a^2}$ and $\Lambda$ defined as in Lemma \ref{lemma:a_q_ratio_bound}, then
\begin{equation} \label{eq_class_corr_bound}
      \max_{\rho \in [0,1]}|R_{\phi}(\rho) - \rho | < \left(\frac{8}{\pi}\right)^{1/3} y^{1/3}
\end{equation}
and
\begin{equation} \label{eq_class_DI_bound}
\left| \frac{\mu_2}{\mu_1^2} - 1 \right| \leq \text{erf}\left( \frac{\Lambda(y)}{\sqrt{2}}\right)^{-2} - 1.
\end{equation}
\end{theorem}
We emphasise that as $y\defeq \sigma_b^2 / a^2 \rightarrow 0$ then both
\[
\max_{\rho \in [0,1]}|R_{\phi,q^*}(\rho) - \rho |, \left|\mu_2/\mu_1^2 - 1 \right| \rightarrow 0.
\]

We first remark that the uniform bound on the correlation provided by Theorem \ref{theorem:main} is only for nonnegative correlations. However, Figure \ref{fig:corr_comp} indicates that letting $\sigma_b^2/a^2 \rightarrow 0$ results $\max_{\rho \in [-1,1]}|R_{\phi,q^*}(\rho) - \rho | \rightarrow 0$. The key takeaway of Theorem \ref{theorem:main} for practitioners is that an improved initialisation can be achieved by using an activation function which has a sufficiently large linear region spanning either side of the origin. The size of this linear region, $a$, needs to be selected on the basis of both the bias variance hyperparameter $\sigma_b^2$ and the depth of the network. The deeper the network or the larger $\sigma_b^2$ is, the larger $a$ needs to be in order to avoid both convergence of the correlations and achieve approximate dynamical isometry. In practice, as per Figure \ref{fig:bound_performance}, modest increases in $a$ typically suffice to capture many of potential benefits at initialisation. We emphasise that there is a tension between the linear region being large enough to achieve a good initialisation, while being small enough so that the expressivity of the network is not reduced. In particular, if the network is initialised so that it acts as linear transform on the input data, and if the optimiser becomes trapped in a local minimum close to the initialisation point, then the network may continue to act as a linear transform throughout training.

\section{Analysis of scaled-bounded activation functions}\label{sec:theory}
\subsection{Derivation of Theorem \ref{theorem:main}}
We start our analysis by revisiting what it means to initialise on the EOC: the condition in \eqref{eq:chi} presupposes the existence of a fixed point $q^*$ of $V_{\phi}$. To resolve this matter we introduce and study the fixed points of a related function, $W$. Before presenting this analysis we introduce the following slight abuse of notation,
\begin{equation} \label{eq:psi}
    \phi'(z):=\frac{1}{2}\left(\frac{d\phi}{dz}(z^-) + \frac{d\phi}{dz}(z^+)\right)
\end{equation}
where $\frac{d\phi}{dz}(z^-)$ and $\frac{d\phi}{dz}(z^+)$ are the left and right derivatives of $\phi$ at $z\in \reals$ respectively. This is convenient for what follows as it allows us to define a notion of a derivative of a scaled-bounded activation $\phi$ over the whole of the domain: we remark, by Definition \ref{def:Omega}, that the true derivative of $\phi$, and $\phi'(\cdot)$, as defined above, are equal almost everywhere.

\begin{lemma} \label{lemma:W}
    Let $\phi$ be a scaled-bounded activation, see Definition \ref{def:Omega}, and $\sigma_b^2>0$. Define
    \begin{equation} \label{eq:W}
        W_{\phi}(q) := \frac{\expec[\phi(\sqrt{q}Z)^2]}{\expec[\phi'(\sqrt{q}Z)^2]} + \sigma_b^2
    \end{equation}
    for all $q \in \reals_{\geq 0}$. Then $W_{\phi}: \reals_{\geq 0} \rightarrow \reals_{\geq 0}$ and $W$ has a fixed point $q^*>0$.
\end{lemma}

A proof of Lemma \ref{lemma:W} is provided in Appendix \ref{appendix:proof_lemma_W}. As a consequence of Lemma \ref{lemma:W}, we are able to make the following claims concerning the existence of fixed points of the variance function $V_{\phi}$ for any scaled-bounded activation $\phi$.

\begin{corollary}\label{cor:V_fixed_point}
    Let $\phi$ be a scaled-bounded activation, see Definition \ref{def:Omega}, $\sigma_b^2>0$ and suppose
    \begin{equation} \label{eq:init_EOC_omega}
        \chi_1 \defeq \sigma_w^2 \expec[\phi'(\sqrt{q^*}Z)^2] =1,
    \end{equation}
    where $q^*>0$ is a fixed point of $W_{\phi}$, defined in (\ref{eq:W}). Then $q^*$ is a fixed point of the associated variance function $V_{\phi}$, defined in (\ref{eq:var_map}).
\end{corollary}

\noindent A proof of corollary \ref{cor:V_fixed_point} is included in Appendix \ref{appendix:proof_V_fixed_point} for completeness. The key takeaway of Corollary \ref{cor:V_fixed_point} is that for any scaled-bounded activation there exists a fixed point $q^*$ of $V_{\phi}$ satisfying $q^*>0$. In fact, as $V_{\phi}(0)= \sigma_b^2>0$ then any fixed point of $V_{\phi}$ is greater than 0. We emphasise that such analysis is necessary as in general care is required when making assumptions concerning the existence of fixed points. For example, the variance function for ReLU like functions do not have any fixed points unless $\sigma_b^2 =0$.

To simplify the analysis we will now proceed under the assumption that the input data is normalised to have euclidean length $q^*$ and that $q^{(l)} = q^*$ for all $l \in [L]$. If one where interested in initialising in this manner in practice, it might also be of interest to explore guarantees concerning the ease with which a stable fixed point of $V_{\phi}$ can be computed: in particular if the fixed point is not attractive then numerical precision issues might mean that the sequence $q^{(l)}$ diverges from $q^*$. We do not pursue this line of inquiry and instead leave it as potential future work. Now that we have identified the existence of fixed points of the variance functions of scaled-bounded activations, it is possible to study the associated correlation functions.

\begin{lemma} \label{lemma:omega_corr_diff}
    Let $\phi$ be a scaled-bounded activation, see Definition \ref{def:Omega}, $\sigma_b^2>0$ and suppose
    \begin{equation}
        \chi_1 \defeq \sigma_w^2 \expec[\phi'(\sqrt{q^*}Z)^2] =1,
    \end{equation}
    where $q^*>0$ is a fixed point of $W_{\phi}$. In addition, assume that all inputs $\vx$ are normalised so that $||\vx||_2^2 = q^*$. Then the associated correlation map $R_{\phi, q*}$, defined in \eqref{eq:corr_map}, is fixed at each layer $l \in [L]$, satisfies $R_{\phi, q*}:[-1,1] \rightarrow [-1,1]$ and is differentiable with
    \begin{equation} \label{eq:omega_corr_map}
        R_{\phi, q^*}'(\rho) = \sigma_w^2 \expec[\phi'(U_1)\phi'(U_2)]
    \end{equation}
    for all input correlations $\rho \in (-1,1)$.
\end{lemma}

A proof of Lemma \ref{lemma:omega_corr_diff} is given in Appendix \ref{appendix:proof_omega_corr_diff}. This expression for the correlation is equivalent to that given in \cite{Poole2016}, however we emphasise, due to the fact that $\phi$ is not necessarily continuously differentiable, that significantly more care is required to derive it. Indeed, if $\phi$ were not odd then this equivalence would not hold and additional terms would appear, potentially complicating the analysis downstream. We are now ready to present a key lemma, which provides a tight uniform bound on the interval $[0,1]$ between the correlation functions associated with scaled-bounded activations and the identity function.

\begin{lemma} \label{lemma:bound_from_identity}
	Under the same conditions and assumptions as in Lemma \ref{lemma:omega_corr_diff}, it holds that
	\begin{equation} \label{eq:bound_from_identity}
    \max_{\rho \in [0,1]}|R_{\phi, q^*}(\rho) - \rho| = \frac{\sigma_b^2}{q^*}.
    \end{equation}
\end{lemma}

A proof of Lemma \ref{lemma:bound_from_identity} is provided in Appendix \ref{appendix:proof_bound_from_identity}. The final lemma we present before proving Theorem \ref{theorem:main} concerns the relationship between $a$, the size of the linear region, and $q^*$, the fixed point of the variance function. This lemma provides lower and upper bounds on the ratio $\frac{a}{\sqrt{q^*}}$ for all scaled-bounded activations as a function of $a$ and $\sigma_b^2$. The fact that this is possible indicates that the particular shape of the tails of a scaled-bounded activation do not play a key role in achieving a good initialisation. In Figure \ref{fig:length_ratio_bounds} we plot these bounds, as well as example $a/\sqrt{q^*}$ ratios, for a number of activation functions as a function of $a$.

\begin{figure}[!ht]
   \begin{minipage}{\textwidth}
     \centering
     \includegraphics[scale=0.5]{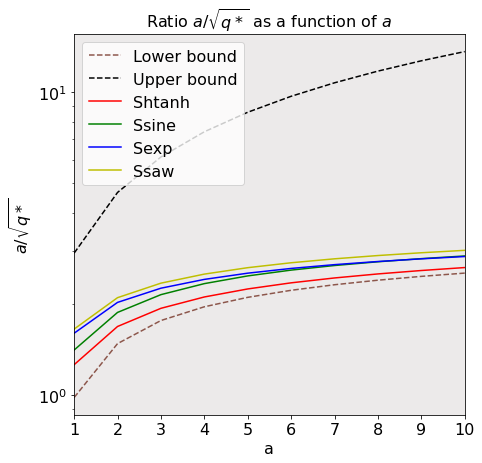}
   \end{minipage}\\[1em]
   \caption{Bounds on $\frac{a}{\sqrt{q^*}}$ given in (\ref{eq:length_ratio_bound}) vs. $\frac{a}{\sqrt{q^*}}$, computed numerically for the same scaled-bounded activations visualised in Figure \ref{fig:omega_examples}.}\label{fig:length_ratio_bounds}
\end{figure}

\begin{lemma}\label{lemma:a_q_ratio_bound}
    Under the same conditions and assumptions as in Lemma \ref{lemma:omega_corr_diff}, and defining $y:=\frac{\sigma_b^2}{a^2}$, then
    \begin{equation} \label{eq:length_ratio_bound}
        \Lambda(y) < \frac{a}{\sqrt{q^*}} < \left(\frac{8}{\pi}\right)^{1/6} y^{-1/3},
    \end{equation}
    where $\Lambda(y)$ is defined as
    \[
    \left(W_0 \left(\frac{2}{\pi}\left(\sqrt{\frac{8}{\pi}}\exp \left(- \frac{\left(\sqrt{W_0 \left(\frac{2}{\pi}y^{-2}\right)}\right)^2}{2}\right) \left( \frac{1}{\left(\sqrt{W_0 \left(\frac{2}{\pi}y^{-2}\right)}\right)} + \left(\sqrt{W_0 \left(\frac{2}{\pi}y^{-2}\right)}\right)\right) \right)^{-2}\right)\right)^{1/2}
    \]
    and $W_0$ denotes the principal branch of the Lambert W function.
\end{lemma}

A proof of Lemma \ref{lemma:a_q_ratio_bound} is given in Appendix \ref{appendix:proof_a_q_ratio_bound}. We note that while the upper bound in Lemma \ref{lemma:a_q_ratio_bound} is easy to interpret, the lower bound is not immediately interpretable. However, this lower bound still allows us to compute a numerical lower bound for any scaled-bounded activation as per Figure \ref{fig:length_ratio_bounds}. In terms of asymptotic behaviour it is easy to check that as $y \rightarrow 0$ then $y^{-1/3}, \Lambda(y) \rightarrow \infty$, and as $y \rightarrow \infty$ then $y^{-1/3}, \Lambda(y) \rightarrow 0$. We also observe from Figure \ref{fig:length_ratio_bounds} that, at least empirically, the lower bound seems to be far tighter: we leave it as potential future work to investigate deriving a tighter upper bound. We are now ready to prove Theorem \ref{theorem:main}.

\begin{proof}
To derive the inequality concerning the correlation function given (\ref{eq_class_corr_bound}) we use the upper bound
\[
\frac{a}{\sqrt{q^*}} < \left(\frac{8}{\pi}\right)^{1/6} \frac{a^{2/3}}{\sigma_b^{2/3}}
\]
derived in Lemma \ref{lemma:a_q_ratio_bound}. Squaring both sides, dividing by $a^2$ and multiplying by $\sigma_b^2$ then
\[
\frac{\sigma_b^2}{q^*} < \left(\frac{8}{\pi}\right)^{1/3} \frac{\sigma_b^{2/3}}{a^{2/3}}.
\]
 To conclude we apply Lemma \ref{lemma:bound_from_identity}. We now turn our attention to proving the inequality concerning the moment ratio provided in equation \ref{eq_class_DI_bound}. Analysing the $l$th moment,
 \[
 \begin{aligned}
 \mu_l &:= \expec [ \phi'(\sqrt{q^*}Z)^{2l}] \\
 &=2 \left( k^{2l} \int_{0}^{a/\sqrt{q^*}} d \gamma(z) + \int_{a/\sqrt{q^*}}^{\infty} \phi'(\sqrt{q^*}z)^{2l}d \gamma(z) \right).
 \end{aligned}
 \]
 By assumption $|\phi'(z)| \leq k$, hence we can bound this quantity as
 \[
 k^{2l} \text{erf}\left(\frac{a}{\sqrt{2q^*}}\right) \leq \mu_l \leq  k^{2l},
 \]
 where the lower bound arises simply by zeroing the second integral term. It therefore follows that
 \[
 \begin{aligned}
\frac{k^4\text{erf}\left(\frac{a}{\sqrt{2q^*}}\right)}{k^4} &\leq \frac{\mu_2}{\mu_1^2} &\leq \frac{k^4}{k^4 \text{erf}\left(\frac{a}{\sqrt{2q^*}}\right)^2},\\
\text{erf}\left(\frac{a}{\sqrt{2q^*}}\right) &\leq \frac{\mu_2}{\mu_1^2} &\leq  \text{erf}\left(\frac{a}{\sqrt{2q^*}}\right)^{-2}\\
\end{aligned}
 \]
Observe that $\mu_2/\mu_1^2 - 1  \leq \text{erf}\left(\frac{a}{\sqrt{2q^*}}\right)^{-2} - 1$ and $ 1 - \mu_2/\mu_1^2 \leq 1 - \text{erf}\left(\frac{a}{\sqrt{2q^*}}\right)$. We now prove, by contradiction, that $\text{erf}\left(\frac{a}{\sqrt{2q^*}}\right)^{-2} - 1 \geq 1 - \text{erf}\left(\frac{a}{\sqrt{2q^*}}\right)$. Indeed, assume that $\text{erf}\left(\frac{a}{\sqrt{2q^*}}\right)^{-2} - 1 < 1 - \text{erf}\left(\frac{a}{\sqrt{2q^*}}\right)$, then
\[
\begin{aligned}
1 - \text{erf}\left(\frac{a}{\sqrt{2q^*}}\right)^{2} &< \text{erf}\left(\frac{a}{\sqrt{2q^*}}\right) \left(\text{erf}\left(\frac{a}{\sqrt{2q^*}}\right) - \text{erf}\left(\frac{a}{\sqrt{2q^*}}\right)^2 \right)\\ &\leq \text{erf}\left(\frac{a}{\sqrt{2q^*}}\right) \left(1 - \text{erf}\left(\frac{a}{\sqrt{2q^*}}\right)^2 \right).
\end{aligned}
\]
As $\text{erf}\left(\frac{a}{\sqrt{2q^*}}\right) <1$ for $\frac{a}{\sqrt{2q^*}} < \infty$ then this is a contradiction. Therefore
\[
\left| \frac{\mu_2}{\mu_1^2} - 1 \right| \leq \text{erf}\left(\frac{a}{\sqrt{2q^*}}\right)^{-2} - 1.
\]
As $\text{erf}$ is monotonically increasing it suffices to lower bound the quantity $\frac{a}{\sqrt{q^*}}$. The result claimed is then recovered by applying the lower bound on $\frac{a}{\sqrt{q^*}}$ derived in Lemma \ref{lemma:a_q_ratio_bound}.
\end{proof}

\begin{figure}[!ht]
  \begin{minipage}{\textwidth}
     \centering
     \includegraphics[scale=.45]{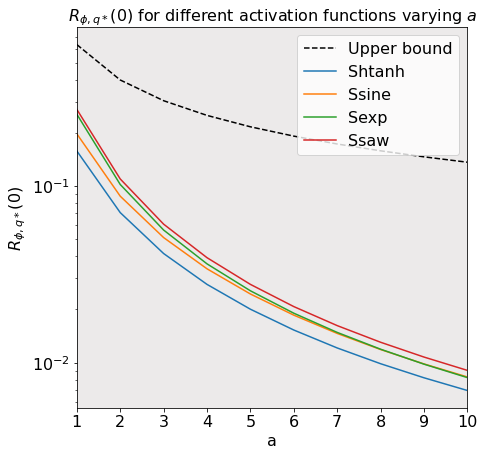}
     \includegraphics[scale=.45]{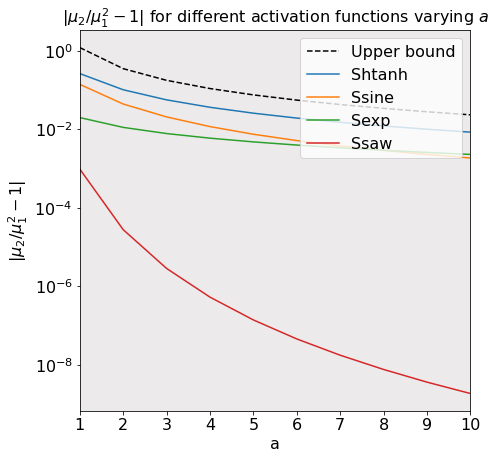}
  \end{minipage}\\[1em]
  \caption{The left and right plots display the correlation and moment ratio bounds, given in Equations (\ref{eq_class_corr_bound}) and (\ref{eq_class_DI_bound}) respectively, vs. the equivalent numerically computed quantities for a variety of different activation functions $\phi \in \Omega$, as shown in Figure \ref{fig:omega_examples}. We note that the right hand curve for hard saw is accurate only up to an $a$ of 6 due to numerical precision issues arising from numerical integration steps.}\label{fig:bound_performance}
\end{figure}

\begin{figure}[!ht]
\centering
     \includegraphics[scale=0.45]{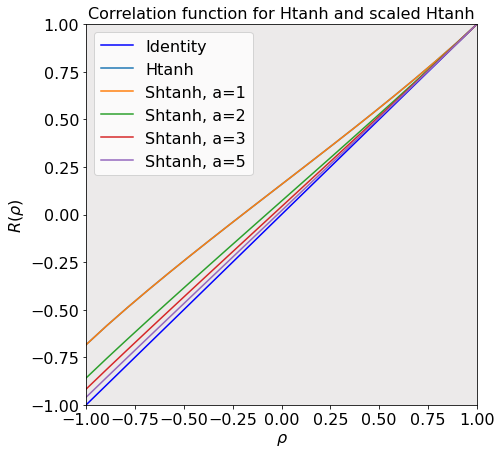}
     \includegraphics[scale=0.45]{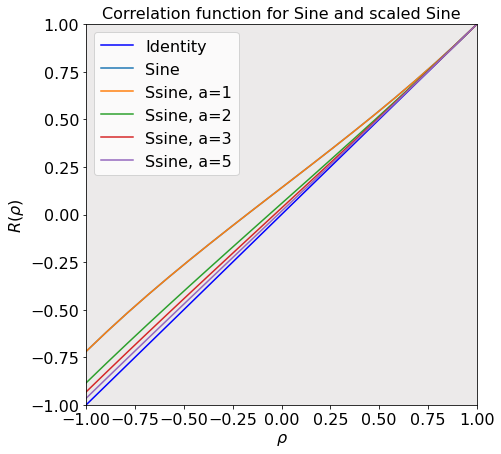}
\caption{impact of scale parameter $a$ on correlation function of scaled (or adapted) htanh (left-hand plot) and sinusoid (right-hand plot).}\label{fig:corr_comp}
\end{figure}

\subsection{Discussion and practical takeaways} \label{subsec:theory_takeaway}
Equation (\ref{eq_class_corr_bound}) implies, by choosing $a$ sufficiently large relative to $\sigma_b^2$, that problems arising as a result of limits on the depth of information propagation as well as network sensitivity can be mitigated without the need for $q^*, \sigma_b^2 \rightarrow 0$. Equation (\ref{eq_class_DI_bound}) likewise implies, as long as orthogonal initialisation is used, that increasing $a$ also mitigates the problem of model degeneracy, again without the need for $q^*, \sigma_b^2 \rightarrow 0$. Our numerics support these conclusions: in Figure \ref{fig:bound_performance} the bounds in (\ref{eq_class_corr_bound}) and (\ref{eq_class_DI_bound}) are plotted numerically with $\sigma_b^2$ fixed, and converge to $0$ as $a$ increases. Likewise, Figure \ref{fig:corr_comp} shows that increasing $a$ moves $R_{\phi, q^*}$ closer to the identity. We remark that (\ref{eq:var_map}) and (\ref{eq:corr_map}) where derived in the context of Gaussian initialisation. As a result one might question whether (\ref{eq_class_corr_bound}) is immediately applicable to orthogonally initialised networks. Indeed, although referenced to in \cite{pennington18a, pennington17a}, the correspondence between orthogonally initialised networks and Gaussian processes is to our knowledge yet to be rigorously established. However, it seems highly likely that the same correspondence holds, an assertion supported both by empirical observation and the fact that large random orthogonal matrices are well approximated by Gaussian matrices (see e.g., \cite{meckes_2019}). We defer a detailed study of this correspondence to later work. We also note that the uniform bound we provide holds for non-negative correlations only, this is due to the fact that the upper bound on this interval, $\sigma_b^2/q^*$ as per Lemma \ref{lemma:bound_from_identity}, is relatively easy to analyse. However, we suspect for scaled-bounded activations that letting $\sigma_b^2/a^2 \rightarrow 0$ results in a uniform convergence of the correlation function to the identity for $\rho \in [-1,1]$. This hypothesis is supported by Figure \ref{fig:corr_comp}.

We conclude this section by considering the relevance and importance of each of the four properties of scaled-bounded activation functions in regard to achieving a good initialisation in practice. Property 1), $\phi$ being continuous, does not seem controversial, indeed most of the commonly used activation functions are continuous. Furthermore, it seems reasonable that continuous activation functions result in a loss landscape that is easier for an optimiser to navigate, compared with that induced by non-continuous ones. Property 2), $\phi$ being odd, is beneficial in theory as it removes additional constant terms from (\ref{eq:bound_from_identity}) in Lemma \ref{lemma:bound_from_identity}. However, as per the proof of Lemma \ref{lemma:bound_from_identity}, it can be observed that these terms will decay exponentially fast as $a$ grows due to the fact that the locations of the points in $\cD$ all have magnitude at least $a$. As for property 3), while the existence of a linear region around the origin is critical to our results, we suspect that boundedness is more an artefact of our proof than a necessity in practice. Assuming that $\phi$ is bounded simplifies certain pieces of analysis, and crucially allows us to formulate upper and lower bounds on $V_{\phi}(q)$, needed for the proof of Lemma \ref{lemma:a_q_ratio_bound}. We hypothesise that this condition could be relaxed to $|\phi(z)| < |z|$. In regard to property 4) and contrary to the conclusion one might be inclined to draw from \cite{hayou19a}, $\phi$ being non-differentiable demonstrates that smoothness is not necessary for a good initialisation. Finally, the bound on the derivative of $\phi$ was introduced to allow us to to derive bounds on $V_{\phi}$.

\section{Experiments} \label{sec:experiments}
\subsection{Experimental setup} \label{subsec:exp_setup}
Across all experiments we train networks with depths $L \in \{20,50,100, 200\}$, with each layer having a fixed width $N=400$, for $100$ epochs on CIFAR-10. The results are averaged over 10 trials. Variance hyperparameters $(\sigma_w^2, \sigma_b^2)$ are selected to lie on the EOC, with $\sigma_b^2\in \{1, 10^{-1}, 10^{-2}, 10^{-3}, 10^{-4}\}$ and $\sigma_w^2$ computed in in order that $\chi_1 =1$, see \eqref{eq:chi}. The parameters of the network are initialised according to the following two schemes.
\begin{itemize}
    \item \textbf{Gaussian initialisation:} all parameters are drawn mutually independent of one another. The weights in each layer are identically distributed with $w_{i,j}^{(l)} \sim \cN(0, \sigma_w^2/N)$. The biases are all identically distributed with $b_{i,j}^{(l)} \sim \cN(0, \sigma_b^2)$.
    \item \textbf{Orthogonal initialisation:} all weight matrices and bias parameters are drawn mutually independent of one another. The weight matrix at each layer is drawn according to the Haar measure, i.e., uniformly, over the orthogonal group of $N\times N$ matrices such that $(\mW^{(l)})^T\mW^{(l)} = \sigma_w^2 \textbf{I}_N$. The biases are all identically distributed with $b_{i,j}^{(l)} \sim \cN(0, \sigma_b^2)$.
\end{itemize}
Finally, optimisation is performed using SGD with a batch size of 64. A learning rate of 10\textsuperscript{-4} was found to be provide good results across all experiments.

\subsection{Experimental validation of the results of Section \ref{sec:theory}}
In order to test and validate the results of Section \ref{sec:theory}, we deploy the scaled Shtanh activation function, defined in \eqref{eq:shtanh},
with $a \in \{1,2,5,10\}$ and $k=1$, and measure the test accuracy of the networks described in Section \ref{subsec:exp_setup} at various stages of training. The results of these experiments are summarised in the heatmaps provided in Figure \ref{fig:cifar_S_htanh}. As per the implications of Theorem \ref{theorem:main}, it is clear from Figure \ref{fig:cifar_S_htanh} that as $\sigma_b^2$ increases a larger but not necessarily maximal value of $a$ gives the best results. This is illustrated by the ridges, indicating the highest test accuracy, running from top left to bottom right in the heatmap subplots. This therefore highlights the trade-off between activation functions which are linear enough to allow for a good initialisation, while not being overly linear that the network loses approximation power. We emphasise that these results are not unique to Shtanh and that the same conclusions can be drawn for other scaled-bounded activation functions \footnote{Implementation and additional results can be found at  \url{https://github.com/Cross-Caps/AFLI}}. As established in prior works, we also observe an advantage in using orthogonal over Gaussian initialisation. However, we also note that this performance gap is relatively small. We leave it to future work to investigate the impact of optimising over, or regularising with respect to, the set of orthogonal weight matrices so that orthogonality is preserved, at least approximately, throughout training.

\begin{figure}[ht!]
\setcounter{figure}{7}
\centering
  \subfloat[Shtanh with Gaussian initialisation]{\includegraphics[trim={0 0 0 0},clip,scale=.5]{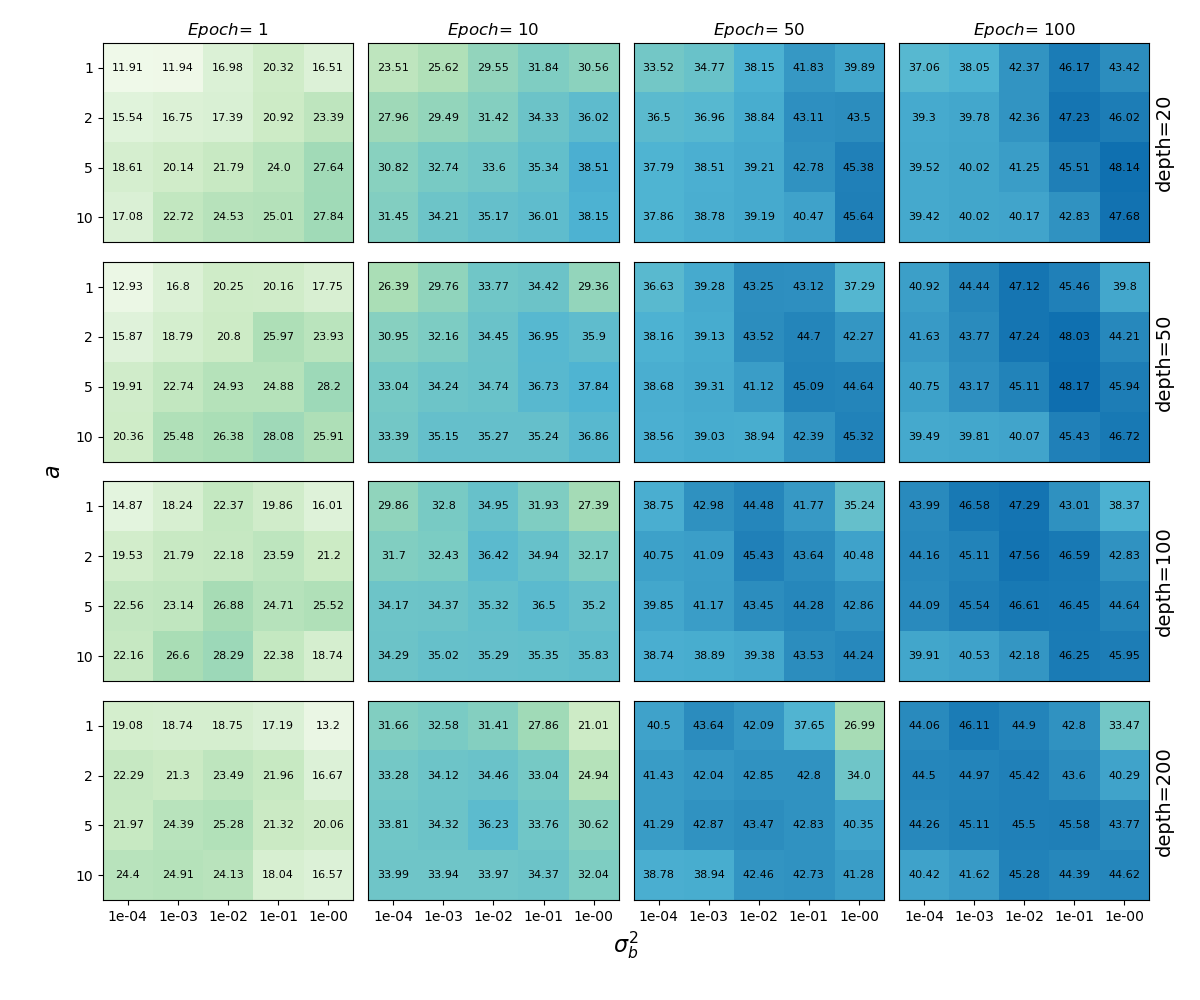}}
  \label{fig:htanh_gauss}
\end{figure} 

\begin{figure}[ht!]
\centering  
\ContinuedFloat
 \subfloat[Shtanh with orthogonal initialisation]{\includegraphics[trim={0 0 0 0},clip,scale=.5]{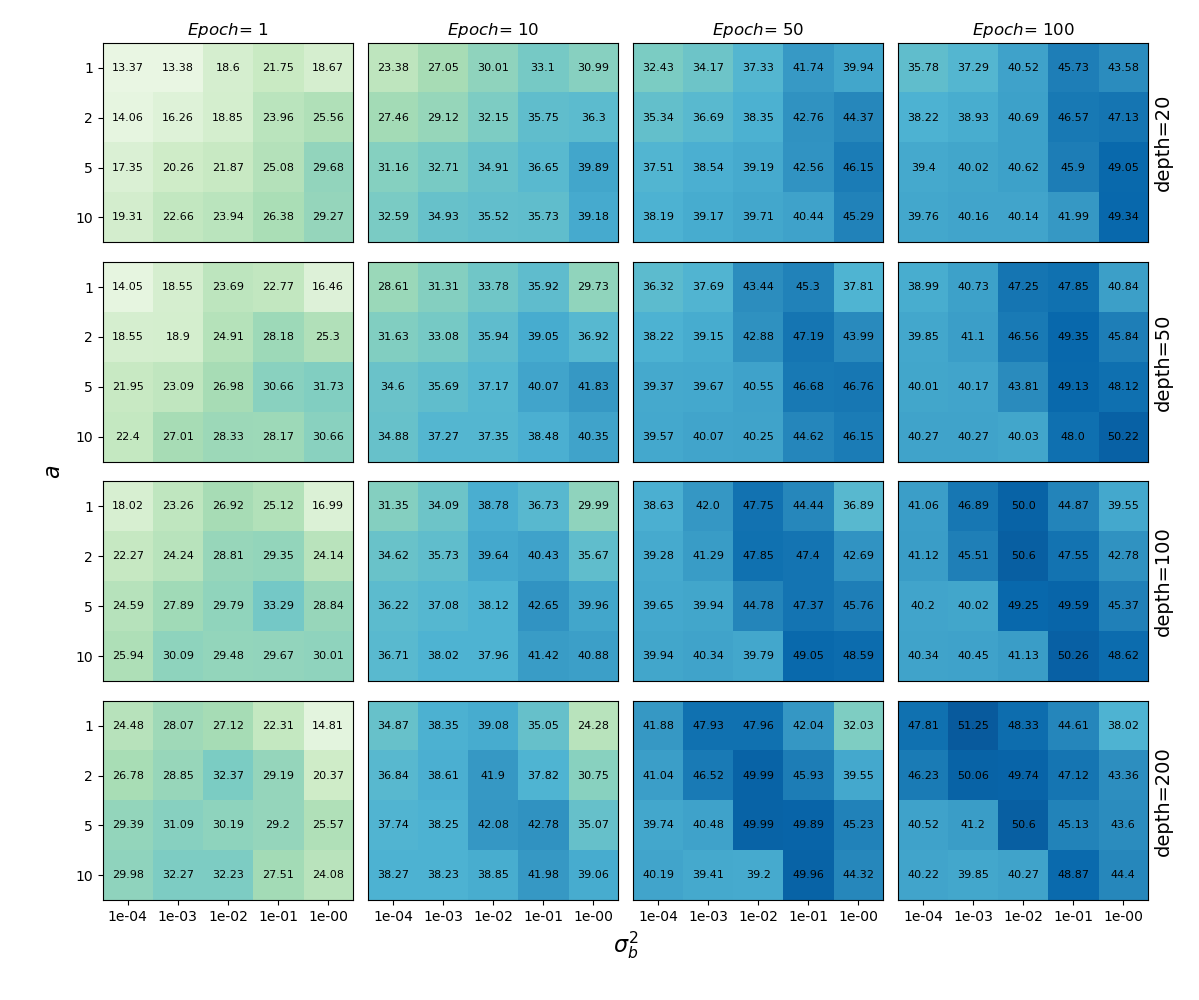}}
 \caption{\small test accuracy on CIFAR-10 at different stages of training. Note that all sub-heatmap plots share the same colour scale.} 
 \label{fig:cifar_S_htanh}
\end{figure}

We also provide some very preliminary results concerning a new training protocol, in which a scaled-bounded activation function is deployed and the value of $a$ is decreased slowly over the first few epochs. To motivate this, observe in Figure \ref{fig:test_error_curves} that, as per our theory, a large value of $a$ always gives the best performance at initialisation. However, for reasons already discussed, an overly large value of $a$ has a negative impact on training long term. This strategy can therefore be interpreted as computing a good initialisation for the final network. The purple curve in Figure \ref{fig:test_error_curves} displays the outcome of this strategy, which appears promising over a wide range of depths and hyperparameter choices.

\begin{figure}[!tbp] 
\setcounter{figure}{8}
\centering
\subfloat[Shtanh with Gaussian initialisation]{\includegraphics[trim={60 0 0 0},clip,scale=.6]{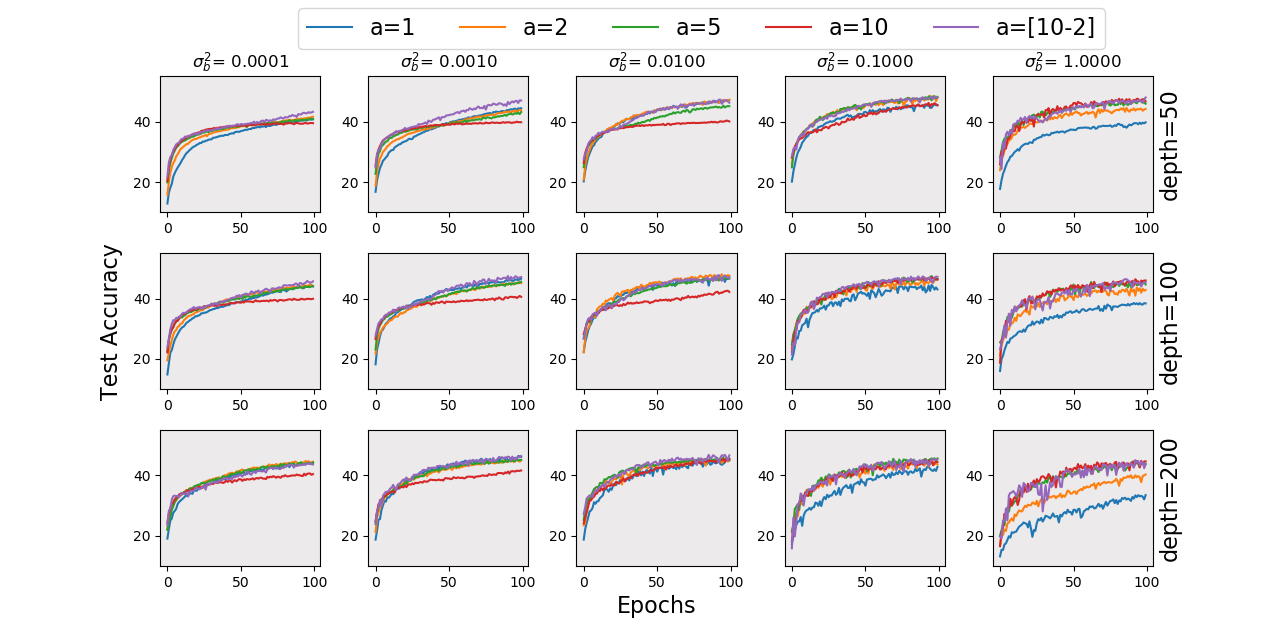}}
\end{figure}

\begin{figure}[!tbp] 
\centering
\ContinuedFloat
\subfloat[Shtanh with orthogonal initialisation]{ \includegraphics[trim={60 0 0 0},clip,scale=.6]{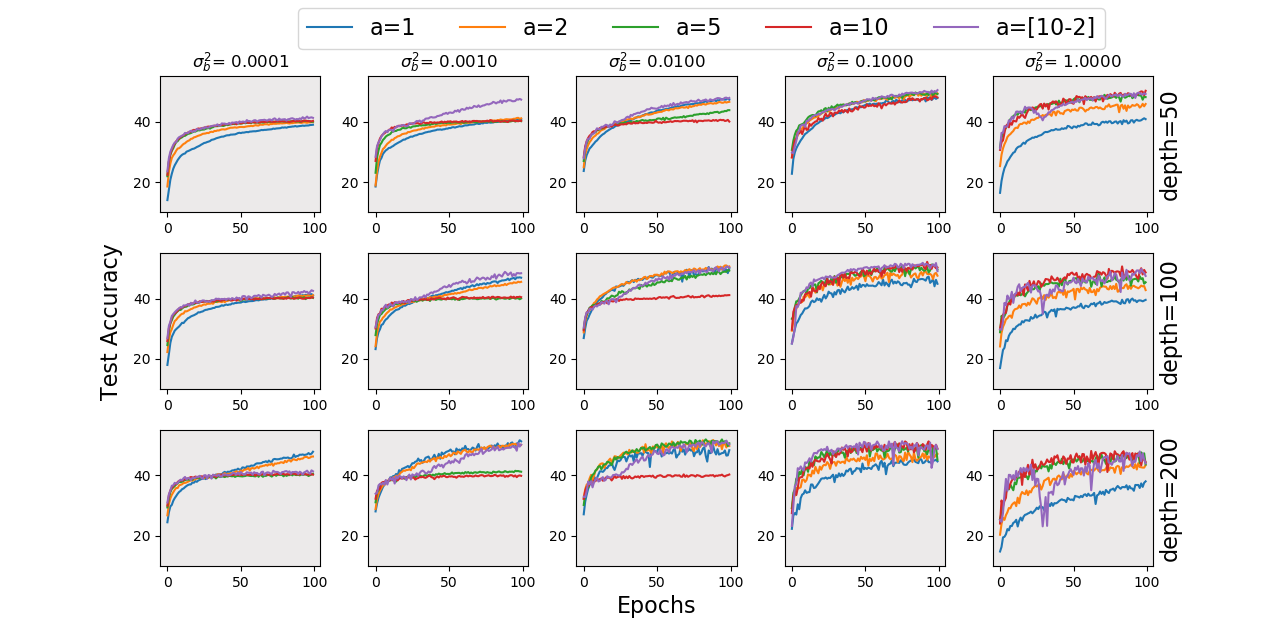}}
 \caption{\small test accuracy during training on CIFAR-10. Networks are trained either with a fixed value of $a$ or with $a$ linearly decreasing from 10$\rightarrow$2 over the first 30 epochs. } \label{fig:test_error_curves}
\end{figure}

\section{Conclusion and avenues for future work} \label{sec:conclusion}
In this paper we considered the role of the activation function in avoiding certain problems at initialisation, namely limited information propagation with depth, high sensitivity or insensitivity to perturbations of the input, and vanishing and exploding gradients. The first two of these we investigated by studying the dynamics of the preactivation correlations, and the third by analysing the spectrum of the input-output Jacobian of the network. Previously, these problems have been ameliorated by shrinking $\sigma_b^2$ in relation to the depth of the network. This is unsatisfactory for two reasons: firstly it results in the expected euclidean length of the activations shrinking towards 0 with depth, and second it places constraints on the initialisation regime, which can result in suboptimal training outcomes. Our theory and experiments clarify that shrinking $\sigma_b^2$ is not necessary, instead, it is possible to avoid these problems at initialisation by ensuring that the activation function deployed has a sufficiently large linear region around the origin. This work therefore provides a rigorous explanation for the observation that activation functions which approximate the identity near the origin perform well, particularly at initialisation.

Avenues for future work include characterising more precisely how large the linear region needs to be for a given depth, a more comprehensive investigation as to the potential benefits of a more affine initialisation, and an analysis of the approximation capabilities of a neural network whose parameters are constrained to lie in some neighbourhood of a given initialisation point. In addition, a quantitative description of how the preactivation correlation dynamics impact training is desirable. 


\section*{Acknowledgements}
This work is supported by the Alan Turing Institute under the EPSRC grant EP/N510129/1 and the Ana Leaf Foundation.

\clearpage
\bibliographystyle{unsrt}  
\bibliography{references}
\clearpage
\appendix
\section{Supporting Lemmas}
For the sake of clarity and completeness we recall here two well known Lemmas concerning Lebesgue integrals. Both can be proved using a combination of Lebesgue's dominated convergence theorem and the mean value theorem (See e.g., Chapter 2 of \cite{folland}). \newline

\begin{lemma} \label{lemma:cont_DCT}
    Let $(X, \cF, \mu)$ be a measure space and $Y\subset \reals$ an open interval. Consider a function $f:X \times Y \rightarrow \reals$ such that the following are true.
    \begin{enumerate}
        \item For each $y \in Y$ then the function $f_y:X \rightarrow \reals$ with $f_y(x) \defeq f(x,y)$ satisfies $f_y(x) \in L^1(X, \cF, \mu)$.
        \item For each $y \in Y$ then $\lim_{y' \rightarrow y} f(x,y') = f(x,y)$ almost everywhere in $X$.
        \item For each $y \in Y$ there exists an open interval $K$, with $y \in K$, and a $g_K(x)\in L^1(X, \cF, \mu)$ such that 
        \[
        |f(x,y')| \leq g_K(x)
        \]
        for all $y' \in K$. 
    \end{enumerate}
    Then the function $F: Y \rightarrow \reals$ with $F(y):= \int_X f(x,y)d\mu$ is continuous on $Y$.\newline
\end{lemma} 

\begin{lemma}
\label{lemma:diff_DCT}
    Let $(X, \cF, \mu)$ be a measure space and $Y\subset \reals$ an open interval. Consider a function $f:X \times Y \rightarrow \reals$ such that the following are true.
    \begin{enumerate}
        \item For each $y \in Y$ then the function $f_y:X \rightarrow \reals$ with $f_y(x) \defeq f(x,y)$ satisfies $f_y(x) \in L^1(X, \cF, \mu)$.
        \item For each $y \in Y$ then $\frac{\partial f}{\partial y}(x,y)$ exists almost everywhere in $X$.
        \item For each $y \in Y$ there exists an open interval $K$, with $y \in K$, and a $g_K(x)\in L^1(X, \cF, \mu)$ such that 
        \[
        |\frac{\partial f}{\partial y}(x,y')| \leq g_K(x)
        \]
        for all $y' \in K$. 
    \end{enumerate}
    Then the function $F:Y \rightarrow \reals$ with $F(y):= \int_X f(x,y)dx$ is differentiable on $Y$ with
    \begin{equation}\label{eq:integral_derivative}
        F'(y) = \int_{X} \frac{\partial f}{\partial y}(x,y)d \mu.
    \end{equation}
\end{lemma}

We also now present a specific adaptation of integration by parts for piecewise continuously differentiable functions of Gaussian random variables with a finite number of discontinuities and bounded derivative. \newline

\begin{lemma} \label{lemma:piecewise_integration_by_parts}
    Suppose $f:\reals \rightarrow \reals$ is bounded, piecewise continuously differentiable at all but a finite number $T \in \naturals$ of non-differentiable points, $t_1< t_2<... <t_T$, and has bounded derivative. Then
    \begin{equation} \label{eq:piecewise_integration_by_parts}
        \int_\reals zf(z) d \gamma(z) = \sum_{i}^{T} \left[-\frac{\exp(-\frac{1}{2}z^2)}{\sqrt{2 \pi}} f(z) \right]_{t_i^+}^{t_{i}^-} +\int_{\reals}f'(z) d \gamma(z)
    \end{equation}
\end{lemma}

\begin{proof}
    We first note that as $f$ and $f'$ are continuous and continuous almost everywhere but at a finite number of points respectively, then they are clearly measureable with respect to the completion of the Borel sigma algebra on $\reals$. Additionally, under the assumptions that $f$ and $f'$ are bounded, it follows that $f(z), zf(z), f'(z) \in L^1(\reals, \cB(\reals), \gamma)$ where $\gamma$ is the standard one dimensional Gaussian measure. Defining, for typographical ease, $g(z)\defeq zf(z)$, $g_{-}(z) \defeq \max\{-zf(z), 0\} $ and $g_{+}(z) \defeq \max\{zf(z), 0\}$, then as $|g_{\pm}(z)|\leq |g(z)|$ it follows that $g_{-}(z), g_{+}(z) \in L^1(\reals, \cB(\reals), \gamma)$. As both $g_{-}$ and $g_{+}$ are nonnegative functions then $(g_{-} \textbf{1}_{[-n,n]})_{n\in \naturals}$ and $(g_{+} \textbf{1}_{[-n,n]})_{n \in \naturals}$ are sequences of non-decreasing functions converging to $g_{-}$ and $g_{+}$ respectively. Therefore, by monotone convergence
    \[
    \begin{aligned}
    \int_{\reals} zf(z) d \gamma(z) &= \int_\reals g_{+}(z) d \gamma(z) - \int_\reals g_{-}(z) d \gamma(z)\\
    &= \lim_{n\rightarrow \infty} \int_{-n}^n g_{+}(z) d \gamma(z) - \lim_{n\rightarrow \infty}\int_{-n}^n g_{-}(z) d \gamma(z)\\
    &= \lim_{n\rightarrow \infty}\int_{-n}^n zf(z) d \gamma(z).
    \end{aligned}
    \]
    Let $n \in \naturals$ be any integer such that $n>|t_T|$. For typographical ease let $t_0 \defeq -n$ and $t_{T+1} \defeq n$, we proceed to analyse the integral of interest over $[-n,n]$.
    \[
    \begin{aligned}
    \int_{-n}^n z f(z) d \gamma(z)
    & = \sum_{i=0}^T\int_{t_{i}}^{t_{i+1}} f(z) z \frac{\exp(-\frac{1}{2}z^2)}{\sqrt{2 \pi}}  dz.
    \end{aligned}
    \]
    By construction $f$ is continuously differentiable on each of the above intervals of integration. Standard integration by parts gives
    \[
    \int_{t_i}^{t_{i+1}} f(z) z \frac{\exp(-\frac{1}{2}z^2)}{\sqrt{2 \pi}}  dz = \left[-\frac{\exp(-\frac{1}{2}z^2)}{\sqrt{2 \pi}} f(z) \right]_{t_i^+}^{t_{i+1}^-} + \int_{t_i}^{t_{i+1}} f'(z) d \gamma(z).
    \]
    Collecting terms it follows that
    \[
    \begin{aligned}
     \int_{-n}^n z f(z) d \gamma(z) &= \sum_{i=0}^{T+1} \left(\left[-\frac{\exp(-\frac{1}{2}z^2)}{\sqrt{2 \pi}} f(z) \right]_{t_i^+}^{t_{i+1}^-} + \int_{t_i}^{t_{i+1}} f'(z) d \gamma(z) \right)\\
     & = \left[-\frac{\exp(-\frac{1}{2}z^2)}{\sqrt{2 \pi}} f(z) \right]_{-n}^{n} + \sum_{i=1}^{T} \left[-\frac{\exp(-\frac{1}{2}z^2)}{\sqrt{2 \pi}} f(z) \right]_{t_i^+}^{t_{i}^-} + \int_{-n}^{n} f'(z) d \gamma(z).
    \end{aligned}
    \]
    Defining $f_{-}'(z)\defeq \max \{-f'(z),0 \}$ and $f_{+}'(z)\defeq \max \{f'(z),0 \}$, then as $|f_{\pm}(z)|\leq |f(z)|$ it follows that $f_{-}(z), f_{+}(z) \in L^1(\reals, \cB(\reals), \gamma)$. As both $f_{-}$ and $f_{+}$ are nonnegative functions, $(f_{-} \textbf{1}_{[-n,n]})_{n\in \naturals}$ and $(f_{+} \textbf{1}_{[-n,n]})_{n \in \naturals}$ are sequences of non-decreasing functions converging to $f^+$ and $f^-$ respectively. Therefore by monotone convergence
    \[
    \begin{aligned}
    \int_{\reals} f'(z) d \gamma(z) &= \int_\reals f_{-}'(z) d \gamma(z) - \int_\reals f_{+}'(z) d \gamma(z)\\
    &= \lim_{n\rightarrow \infty} \int_{-n}^n f_{-}'(z) d \gamma(z) - \lim_{n\rightarrow \infty}\int_{-n}^n f_{+}'(z) d \gamma(z)\\
    &= \lim_{n\rightarrow \infty}\int_{-n}^n f'(z) d \gamma(z).
    \end{aligned}
    \]
    As a result
    \[
    \begin{aligned}
    \int_{\reals} z f(z) d \gamma(z) & = \lim_{n \rightarrow \infty} \int_{-n}^n z f(z) d \gamma(z)\\
    & = \lim_{n \rightarrow \infty}\left[-\frac{\exp(-\frac{1}{2}z^2)}{\sqrt{2 \pi}} f(z) \right]_{-n}^{n} + \sum_{i=1}^{T} \left[-\frac{\exp(-\frac{1}{2}z^2)}{\sqrt{2 \pi}} f(z) \right]_{t_i^+}^{t_{i}^-} + \lim_{n \rightarrow \infty}\int_{-n}^{n} f'(z) d \gamma(z)\\
    & = \sum_{i=1}^{T} \left[-\frac{\exp(-\frac{1}{2}z^2)}{\sqrt{2 \pi}} f(z) \right]_{t_i^+}^{t_{i}^-} +\int_{\reals}f'(z) d \gamma(z)
    \end{aligned}
    \]
    as claimed. 
\end{proof}


\section{Section \ref{sec:theory} proofs}
\subsection{Lemma \ref{lemma:W}} \label{appendix:proof_lemma_W}
\textbf{Lemma \ref{lemma:W}.} \textit{Let $\phi$ be a scaled-bounded activation, see Definition \ref{def:Omega}, and $\sigma_b^2>0$. Define
    \[
        W_{\phi}(q) := \frac{\expec[\phi(\sqrt{q}Z)^2]}{\expec[\phi'(\sqrt{q}Z)^2]} + \sigma_b^2
    \]
    for all $q \in \reals_{\geq 0}$. Then $W_{\phi}: \reals_{\geq 0} \rightarrow \reals_{\geq 0}$ and $W$ has a fixed point $q^*>0$.}

\begin{proof}
    We first prove that $W_{\phi}: \reals_{\geq 0} \rightarrow \reals_{\geq 0}$. To bound the numerator term then for any scaled-bounded activation function $\phi$, see Definition \ref{def:Omega}, it follows that there exists $a,k\in \reals{>0}$ such that
    \begin{equation} \label{eq:phi_expec_bound}
        0 \leq \expec[\phi(\sqrt{q}Z)^2] < a^2 k^2 < \infty.
    \end{equation}
    for all $q \in \reals_{\geq 0}$. As $\expec[\phi'(\sqrt{q}Z)^2] \geq 0$ it suffices to show that the denominator is nonzero. For a given $q\in \reals_{\geq 0}$. as $\phi'(\sqrt{q}z)^2\geq 0$ for all $z \in \reals \backslash \cD$, then $\expec[\phi'(\sqrt{q}Z)^2]=0$ implies that $\phi'(\sqrt{q}z)=0$ almost everywhere for $z \in \reals$. This is a contradiction however as there exists by construction an $a>0$ such that $\phi(\sqrt{q}z) = kz$ for $z\in [-a/\sqrt{q},a/\sqrt{q}]$. As a result $\expec[\phi'(\sqrt{q}Z)^2] > 0$. We therefore conclude that $W_{\phi}: \reals_{\geq 0} \rightarrow \reals_{\geq 0}$.
    
    To prove that $W_{\phi}$ has a fixed point $q^*>0$ we need to lower and upper bound $\expec[\phi'(\sqrt{q}Z)^2]$. A lower bound can be derived as follows,
    \[
    \begin{aligned}
    \expec[\phi'(\sqrt{q}Z)^2] &= \int_{\reals} \phi'(\sqrt{q}Z)^2 d \gamma\\
    & = 2\int_{0}^{\infty} \phi'(\sqrt{q}Z)^2 d \gamma\\
    & = 2\int_{0}^{a/ \sqrt{q}} k^2 d \gamma + 2\int_{a/ \sqrt{q}}^{\infty} \phi'(\sqrt{q}Z)^2  d \gamma\\
    &> 2 k\int_0^{a/ \sqrt{q}} \phi'(\sqrt{q}Z)^2 d \gamma\\
    & = k^2\text{erf}\left( \frac{a}{\sqrt{2q}}\right).
    \end{aligned}
    \]
    Here the second equality follows from the fact that integrand is an even function and the equality on the third line by the construction of $\phi$. The inequality on the fourth line follows from zeroing the second integral which is positive. Furthermore $|\phi'(z)|\leq k$ almost everywhere by construction, therefore we conclude that
    \begin{equation}\label{eq:phi_diff_expec_bound}
        k^2\text{erf}\left( \frac{a}{\sqrt{2q}}\right) \leq \expec[\phi'(\sqrt{q}Z)^2] < k^2 < \infty
    \end{equation}
    for all $q \in \reals_{0}$. 
    
    We now prove that $W_{\phi}$ is continuous. Due to fact that  $\expec[\phi'(\sqrt{q}Z)^2] > 0$ then $W_{\phi}$ has no singularities. It therefore suffices to prove that both $q \mapsto \expec[\phi(\sqrt{q}Z)^2]$ and $q \mapsto \expec[\phi'(\sqrt{q}Z)^2]$ are continuous functions on $\reals_{\geq 0}$. In both cases we achieve this by applying Lemma \ref{lemma:cont_DCT} in the context of the measure space $(\reals, \cB(\reals), \gamma)$, where $\gamma$ denotes the standard one dimensional Gaussian measure and $\cB(\reals)$ the completion of the Borel $\sigma$ algebra on the real numbers. For $q \mapsto \expec[\phi(\sqrt{q}Z)^2]$ with $q \in \reals_{\geq0}$ then condition 1 is satisfied due to \eqref{eq:phi_expec_bound}, condition 2 follows from the continuity of $\phi$ and condition 3 is satisfied by $g(z) \defeq a^2k^2$. For $q \mapsto \expec[\phi'(\sqrt{q}Z)^2]$ with $q \in \reals_{\geq0}$ then condition 1 is satisfied due to \eqref{eq:phi_diff_expec_bound}, condition 2 follows from the fact that $\phi'$ is continuous almost everywhere in $\reals$ and condition 3 is satisfied for all $q \in \reals_{\geq 0}$ by $g(z): = k^2$. We conclude then that $W_{\phi}$ is continuous on $\reals_{\geq 0}$.

    To prove that $0$ is not a fixed point, observe that $\frac{\expec[\phi(\sqrt{q}Z)^2]}{\expec[\phi'(\sqrt{q}Z)^2]}\geq 0$ and therefore $W_{\phi}(q) \geq \sigma_b^2>0$ for all $q \in \reals_{\geq0}$. Using \eqref{eq:phi_expec_bound} and \eqref{eq:phi_diff_expec_bound} we may in addition derive the following upper bound on $W_{\phi}$,
    \begin{equation} \label{eq:W_bounds}
        \sigma_b^2 \leq W_{\phi}(q) < \frac{a^2}{\text{erf}\left( \frac{a}{\sqrt{2q}}\right)} + \sigma_b^2 \eqdef U(q)
    \end{equation}
    for all $q \in \reals_{\geq 0}$. Observe that as $W_{\phi}$ is continuous and $W_{\phi}(0)>0$, then if $W_{\phi}$ has no fixed points it must hold that $W_{\phi}(q) > q$ for all $q \in \reals_{\geq 0}$. Otherwise, by the intermediate value theorem, the function $W_{\phi}(p) - p$ must have a root and hence $W$ must have a fixed point $q^*>0$. Considering the upper bound $U$ on $W_{\phi}$ from equation \ref{eq:W_bounds}, assume that $U(q)>q$ for all $q\in \reals_{\geq 0}$. Then
    \[
    \begin{aligned}
    \frac{a^2}{\text{erf}\left( \frac{a}{\sqrt{2q}}\right)} + \sigma_b^2&>q,\\
     a^2 + \sigma_b^2\text{erf}\left( \frac{a}{\sqrt{2q}} \right) &> q \text{erf}\left( \frac{a}{\sqrt{2q}} \right)\\
    \end{aligned}
    \]
    However, $\lim_{q \rightarrow \infty}  a^2 + \sigma_b^2\text{erf} \left( \frac{a}{\sqrt{2q}} \right)  = a^2 < \infty$ while $\lim_{q \rightarrow \infty}q \text{erf}\left( \frac{a}{\sqrt{2q}} \right) = \infty$. As $q \text{erf}\left( \frac{a}{\sqrt{2q}} \right)$ is continuous then there must exist a $q \in \reals_{\geq 0}$ such that $U(q)<q$, which is a contradiction. We therefore conclude that $W_{\phi}(q)<U(q)<q$ for some $q\in \reals_{\geq 0}$ and that therefore $W_{\phi}$ must have a fixed point $q^*>0$.
\end{proof}

\subsection{Corollary \ref{cor:V_fixed_point}} \label{appendix:proof_V_fixed_point}
\textbf{Corollary  \ref{cor:V_fixed_point}.} \textit{  Let $\phi$ be a scaled-bounded activation, see Definition \ref{def:Omega}, $\sigma_b^2>0$ and suppose
    \[
        \chi_1 \defeq \sigma_w^2 \expec[\phi'(\sqrt{q^*}Z)^2] =1,
    \]
    where $q^*>0$ is a fixed point of $W_{\phi}$, defined in (\ref{eq:W}). Then $q^*$ is a fixed point of the associated variance function $V_{\phi}$, defined in (\ref{eq:var_map}).}
    
\begin{proof}
    From Lemma \ref{lemma:W} it holds that there exists $q^*>0$ such that $W_{\phi}(q^*) = q^*$. Inspecting \eqref{eq:var_map} and \eqref{eq:W} it follows that $V_{\phi}(q^*) = W_{\phi}(q^*) = q^*$, therefore $q^*$ is a fixed point of $V_{\phi}$.
\end{proof}

\subsection{Lemma \ref{lemma:omega_corr_diff}} \label{appendix:proof_omega_corr_diff}
\textbf{Lemma \ref{lemma:omega_corr_diff}.} \textit{ Let $\phi$ be a scaled-bounded activation, see Definition \ref{def:Omega}, $\sigma_b^2>0$ and suppose
    \[
        \chi_1 \defeq \sigma_w^2 \expec[\phi'(\sqrt{q^*}Z)^2] =1,
    \]
    where $q^*>0$ is a fixed point of $W_{\phi}$. In addition, assume that all inputs $\vx$ are normalised so that $||\vx||_2^2 = q^*$. Then the associated correlation map $R_{\phi, q*}$, defined in \eqref{eq:corr_map}, is fixed at each layer $l \in [L]$, satisfies $R_{\phi, q*}:[-1,1] \rightarrow [-1,1]$ and is differentiable with
    \[
        R_{\phi, q^*}'(\rho) = \sigma_w^2 \expec[\phi'(U_1)\phi'(U_2)]
    \]
    for all input correlations $\rho \in (-1,1)$.}

\begin{proof}
    Recall that $U_1 := \sqrt{q^*} Z_1$ and $U_2 := \sqrt{q^*}(\rho Z_1 + \sqrt{1- \rho^2}Z_2)$. From Corollary \ref{cor:V_fixed_point} we know that $q^*$ is a fixed point of $V_{\phi}$ and therefore the variance of all inputs, given the assumed normalisation, will remain fixed at $q^*$ for all layers of the network. Since $q^*>0$, then $R_{\phi, q*}:[-1,1] \rightarrow [-1,1]$ by construction as long as the correlation function \eqref{eq:corr_map} is finite for any $\rho \in [-1,1]$. This follows by Cauchy-Schwarz,
     \begin{equation} \label{eq:H_arg_integrable}
     \begin{aligned}
        |\expec[\phi(U_1)\phi(U_2)]| \leq \expec[\phi(U_1)^2]^{1/2}\expec[\phi'(U_2)^2]^{1/2}
         < a^2k^2 < \infty.
    \end{aligned}
     \end{equation}
     It remains to be proved that $R_{\phi, q^*}$ is differentiable on $(-1,1)$ and to derive \eqref{eq:omega_corr_map}.  To this end it suffices to show that $H(\rho) \defeq \expec[\phi(U_1)\phi(U_2)]$ is differentiable on $(-1,1)$ and derive an expression for its derivative. We rewrite $H$ as follows,
    \[
     \begin{aligned}
        H(\rho) &\defeq \int_{\reals \times \reals}\phi(u_1) \phi(u_2) d \gamma^{(2)}(z_1,z_2)
     \end{aligned}
     \]
     recalling that $\gamma^{(2)}$ denotes the standard two dimensional Gaussian measure. We proceed by applying Lemma \ref{lemma:diff_DCT} in the context of the measure space $(\reals^2, \cB(\reals^2), \gamma^2)$, where $\cB(\reals^2)$ denotes the completion of the Borel $\sigma$-algebra on $\reals^2$, and the interval $(-1,1)$. First observe that condition 1 of Lemma \ref{lemma:diff_DCT} is satisfied as
     \[
     \int_{\reals \times \reals} |\phi(u_1) \phi(u_2)| d \gamma^{(2)}(z_1,z_2) < a^2k^2 < \infty.
     \]
     For condition 2, by construction $(\phi \circ u_2)(\rho, z_1, z_2)$ is non-differentiable only on the set
     \[
     \bigcup_{i=1}^{|\cD|} \{(z_1,z_2): \rho z_1 + \sqrt{1 - \rho^2}z_2 = \frac{d_i}{\sqrt{q^*}}\}.
     \]
    This is the union of a finite number of one dimensional lines in $\reals^2$ and hence has measure 0. Therefore, for each $\rho \in (-1,1)$, then $\frac{\partial \phi \circ u_2}{(\partial \rho)}(x,z_1,z_2)$ exists almost everywhere and hence condition 2 is also satisfied. Finally, for condition 3 note that this partial derivative can be expressed as
    \[
    \begin{aligned}
        \frac{\partial \phi \circ u_2}{\partial \rho}(\rho,z_1,z_2) &= \frac{\partial u_2}{\partial \rho} (\rho, z_1,z_2) \phi'(u_2)\\
        &=\sqrt{q^*}\left( z_1 - \frac{\rho}{\sqrt{1 - \rho^2}}z_2\right)\phi'(u_2),
    \end{aligned}
    \]
    from which it follows
    \[
        \begin{aligned}
        \left| \frac{\partial \phi \circ u_2}{\partial \rho}(\rho,z_1,z_2) \right| &\leq \sqrt{q^*}\left( |z_1| + \frac{|\rho|}{\sqrt{1 - \rho^2}}|z_2| \right) |\phi'(u_2)| \\
        & < \sqrt{q^*}k\left( |z_1| + \frac{|\rho|}{\sqrt{1 - \rho^2}}|z_2| \right).
        \end{aligned}
    \]
     For any $\rho\in (-1,1)$ consider the open interval $(-1+\delta(\rho), 1- \delta(\rho))$, where $\delta(\rho):= \frac{|1-\rho|}{2}$. It follows that 
    \[
    \sup_{\rho \in (-1+\delta(\rho), 1- \delta(\rho))} \left| \frac{\partial \phi \circ u_2}{\partial \rho}(\rho,z_1,z_2) \right| \leq \sqrt{q^*}k \left(|z_1| +\frac{1-\delta(\rho)}{\sqrt{2\delta(\rho)-\delta(\rho)^2}} |z_2|\right) \eqdef g_K(z_1,z_2).
    \]
    Letting $\kappa_{\delta} \defeq  \frac{1-\delta(\rho)}{\sqrt{2\delta(\rho)-\delta(\rho)^2}}$, then by applying the Fubini-Tonelli theorem it follows that
    \[
    \begin{aligned}
    \int_{\reals^2} g_K(z_1,z_2) d \gamma^{(2)}(z_1,z_2)  & = \sqrt{q^*}k \left( \int_{\reals^2} |z_1| d \gamma^{(2)}(z_1,z_2) +  \kappa_{\delta} \int_{\reals^2} |z_2| d\gamma^{(2)}(z_1,z_2) \right)\\
    &=\sqrt{q^*}k \sqrt{\frac{2}{\pi}}\left(1 + \kappa_{\delta} \right)\\
    &< \infty.
    \end{aligned}
    \]
    Hence for any $\rho \in (-1,1)$ there exists an open interval $K=(-1+\delta(\rho), 1- \delta(\rho))$ with $\rho \in K$ and an integrable function $g_{K}(z_1, z_2):\reals^2 \rightarrow \reals$, such that
    \[
        \left|\frac{\partial (\phi \circ u_2)}{\partial \rho}(\rho,z_1,z_2) \right| \leq  g_{K}(z_1, z_2).
    \]
    We conclude then that condition 3 is also satisfied, and therefore, by Lemma \ref{lemma:diff_DCT}, $H$ is differentiable on $(-1,1)$. The derivative of $H$ can be expressed as follows,
    \begin{equation}\label{eq:H_diff_intermediate_result}
    \begin{aligned}
         H'(\rho) &= \int_{\reals^2} \phi(u_1)  \frac{\partial u_2}{\partial \rho}\phi'(u_2) d\gamma^{(2)}(z_1,z_2)\\
         & = \sqrt{q^*} \int_{\reals}\int_{\reals} \phi(u_1)\left( z_1 - \frac{\rho}{\sqrt{1 - \rho^2}}z_2\right)\phi'(u_2)d\gamma(z_1) d\gamma(z_2)\\
         & = \sqrt{q^*} \int_{\reals} \left( \int_{\reals} z_1\phi(u_1)\phi'(u_2)d \gamma(z_1)\right) d \gamma(z_2)
         - \frac{\rho\sqrt{q^*}}{\sqrt{1 - \rho^2}} \int_{\reals}\phi(u_1)\left(\int_{\reals} z_2 \phi'(u_2)d\gamma(z_2) \right) d\gamma(z_1),
    \end{aligned}
    \end{equation}
    where the third equality in the above follows by applying the Fubini-Tonelli theorem. We proceed first to derive an expression for $H'(0)$ and then to analyse $H'(\rho)$ for $\rho \in (-1,0)\cup (0,1)$. In all that follows, and as $\phi$ is odd, we let $-d_T...<-d_2<-d_1<d_1<d_2..<d_T$ be the elements of $\cD$ where $T \defeq |\cD|/2$. From \eqref{eq:H_diff_intermediate_result} it follows that
    \[
    \begin{aligned}
       H'(0) &= \sqrt{q^*} \int_{\reals} \left( \int_{\reals} z_1\phi(\sqrt{q^*}z_1)\phi'(\sqrt{q^*z_2})d\gamma(z_1)\right) d\gamma(z_2)\\
       &= \sqrt{q^*} \int_{\reals} \phi'(\sqrt{q^*z_2}) \left( \int_{\reals} z_1\phi(\sqrt{q^*}z_1)d\gamma(z_1)\right) d\gamma(z_2)
    \end{aligned}
    \]
    By construction $\phi(\sqrt{q^*}z_1)$ is differentiable at all points other than
    \[
    \begin{aligned}
    & s_i \defeq \frac{d_i}{\sqrt{q^*}},\\
    & l_i \defeq -\frac{d_i}{\sqrt{q^*}}
    \end{aligned}
    \]
    for $i \in [T]$. Applying Lemma \ref{lemma:piecewise_integration_by_parts}
    \[
    \begin{aligned}
        \int_{\reals} z_1\phi(\sqrt{q^*}z_1)d\gamma(z_1) &= - \sum_{i}^{T} \left( \left[\frac{\exp(-\frac{1}{2}z_1^2)}{\sqrt{2 \pi}} \phi(\sqrt{q^*}z_1)\right]_{l_i^+}^{l_{i}^-} + \left[\frac{\exp(-\frac{1}{2}z_1^2)}{\sqrt{2 \pi}} \phi(\sqrt{q^*}z_1) \right]_{s_i^+}^{s_{i}^-}\right)\\ 
        &+ \sqrt{q^*}\int_{\reals}\phi'(\sqrt{q^*}z_1) d \gamma(z_1)\\
        & =\sqrt{q^*} \int_{\reals}\phi'(\sqrt{q^*}z_1) d \gamma(z_1).
    \end{aligned}
    \]
    Here the second equality follows from the continuity of $\phi$. Therefore
    \begin{equation}
    \begin{aligned}
       H'(0) &=  \sqrt{q^*} \int_{\reals} \phi'(\sqrt{q^*z_2}) \left( \sqrt{q^*} \int_{\reals}\phi'(\sqrt{q^*}z_1) d \gamma(z_1)\right) d\gamma(z_2)\\
       &= q^* \int_{\reals^2} \phi'(\sqrt{q^*}z_1)\phi'(\sqrt{q^*}z_2) d \gamma^{(2)}(z_1,z_2)
    \end{aligned}
    \end{equation}

    To derive an expression for $H'(\rho)$ for any $\rho \in (-1,0)\cup (0,1)$, we apply integration by parts to each of the inner integrals in \eqref{eq:H_diff_intermediate_result} using Lemma \ref{lemma:piecewise_integration_by_parts}. Starting with the second inner integral, with $\rho \in (-1,0)\cup(0,1)$ and $z_1 \in \reals$ fixed, we define $\psi:\reals \rightarrow \reals$ as the function $z_2 \mapsto (\phi' \circ u_2)(z_1,z_2, \rho)$. By construction $\phi'$ is continuously differentiable at all points other than the following,
    \[
    \begin{aligned}
    p_i(z_1) &\defeq \frac{d_i}{\sqrt{q^*(1 - \rho^2)}} - z_1 \frac{ \rho}{\sqrt{1-\rho^2}},\\
    n_i(z_1) &\defeq -\frac{d_i}{\sqrt{q^*(1 - \rho^2)}} - z_1 \frac{ \rho}{\sqrt{1-\rho^2}}
    \end{aligned}
    \]
    for all $i \in [T]$. Applying Lemma \ref{lemma:piecewise_integration_by_parts} it follows that
    \[
    \begin{aligned}
    \int_{\reals} z_2 \psi(z_2)d\gamma(z_2) &= - \sum_{i}^{T} \left( \left[\frac{\exp(-\frac{1}{2}z_2^2)}{\sqrt{2 \pi}} \psi(z_2) \right]_{n_i^+}^{n_{i}^-} + \left[\frac{\exp(-\frac{1}{2}z_2^2)}{\sqrt{2 \pi}} \psi(z_2) \right]_{p_i^+}^{p_{i}^-}\right) +\int_{\reals}\psi'(z_2) d \gamma(z_2)\\
    & \eqdef - \kappa_1(z_1) + \int_{\reals}\psi'(z_2) d \gamma(z_2).
    \end{aligned}
    \]
    Therefore the second integral on the final line of \eqref{eq:H_diff_intermediate_result} can be expressed as
    \[
    \begin{aligned}
    &\int_{\reals}\phi(u_1)\left(\int_{\reals} z_2 \phi'(u_2)d\gamma(z_2) \right) d\gamma(z_1) \\
    & = -\int_{\reals}\phi(u_1) \kappa_1(z_1)d\gamma(z_1) + \int_{\reals}\phi(u_1) \int_{\reals}\psi'(z_2) d \gamma(z_2)d\gamma(z_1)\\
    &= -\int_{\reals}\phi(u_1) \kappa_1(z_1)d\gamma(z_1) + \sqrt{q^*(1 - \rho^2)}\int_{\reals^2}\phi(u_1) \phi''(u_2) d \gamma^{(2)}(z_1,z_2).
    \end{aligned}
    \]
    Analysing $\kappa_1(z_1)$, then by construction $\psi(n_i) = \phi'(-d_i)$ and $\psi(p_i) = \phi'(d_i)$. Furthermore, as $\phi$ is odd and continuous then $\phi'$ is even and as a result $\phi'(-d_i^-) = \phi'(d_i^+)$ and $\phi'(-d_i^+) = \phi'(d_i^-)$. Therefore
    \[
    \begin{aligned}
    \kappa_1(z_1) &= \sum_{i}^{T} \left( \left[\frac{\exp(-\frac{1}{2}z^2)}{\sqrt{2 \pi}} \psi(z) \right]_{n_i^+}^{n_{i}^-} + \left[\frac{\exp(-\frac{1}{2}z^2)}{\sqrt{2 \pi}} \psi(z) \right]_{p_i^+}^{p_{i}^-}\right)\\
    & = \sum_{i}^{T} \left(\frac{\exp(-\frac{1}{2}n_i^2)}{\sqrt{2 \pi}} \left(\psi(n_i^-) - \psi(n_i^+)  \right) + \frac{\exp(-\frac{1}{2}p_i^2)}{\sqrt{2 \pi}} \left(\psi(p_i^-) - \psi(p_i^+)  \right) \right)\\
    & = \sum_{i}^{T} \left(\frac{\exp(-\frac{1}{2}n_i^2)}{\sqrt{2 \pi}} \left(\phi'(-d_i^-) - \phi'(-d_i^+)  \right) + \frac{\exp(-\frac{1}{2}p_i^2)}{\sqrt{2 \pi}}  \left(\phi'(d_i^-) - \phi'(d_i^+)  \right) \right)\\
    & = \sum_{i}^{T} \left(\frac{\exp(-\frac{1}{2}n_i^2)}{\sqrt{2 \pi}} \left(\phi'(d_i^+) - \phi'(d_i^-)  \right) + \frac{\exp(-\frac{1}{2}p_i^2)}{\sqrt{2 \pi}}  \left(\phi'(d_i^-) - \phi'(d_i^+)  \right) \right)\\
    & = \frac{1}{\sqrt{2 \pi}}\sum_{i}^{T} \left(\phi'(d_i^+) - \phi'(d_i^-)  \right) \left(\exp(-\frac{1}{2}n_i^2)- \exp(-\frac{1}{2}p_i^2)\right).
    \end{aligned}
    \]
    Expanding the integral involving $\kappa_1$ then
    \[
    \begin{aligned}
    \int_{\reals}\phi(u_1) \kappa_1(z_1)d\gamma(z_1) = \sum_{i}^{T} \left(\phi'(d_i^+) - \phi'(d_i^-)  \right) \int_{\reals} \phi(u_1)\frac{\exp(-\frac{1}{2}n_i^2)- \exp(-\frac{1}{2}p_i^2)}{\sqrt{2 \pi}}d \gamma(z_1).
    \end{aligned}
    \]
    Observe that as
    \[
    \begin{aligned}
    &p_i(z_1)^2 = \frac{d_i^2}{q^*(1 - \rho^2)} - z_1 \frac{2d_i\rho}{\sqrt{q^*}(1 - \rho^2)} + z_1^2 \frac{\rho^2}{1- \rho^2}\\
    &n_i(z_1)^2 = \frac{d_i^2}{q^*(1 - \rho^2)} + z_1 \frac{2d_i\rho}{\sqrt{q^*}(1 - \rho^2)} + z_1^2 \frac{\rho^2}{1- \rho^2}
    \end{aligned}
    \]
    then clearly $n_i(z_1)^2 = p_i(-z_1)^2$. Let $\beta_i(z_1) \defeq e^{-\frac{1}{2}n_i(z_1)^2} - e^{-\frac{1}{2}p_i(z_1)^2}$ for all $i \in [T]$, then these $\beta_i$ are odd functions as
    \[
    \begin{aligned}
    \beta_i(-z_1) &= e^{-\frac{1}{2}n_i(-z_1)^2} - e^{-\frac{1}{2}p_i(-z_1)^2}\\
    & = e^{-\frac{1}{2}p_i(z_1)^2} - e^{-\frac{1}{2}n_i(z_1)^2}\\
    & = -\beta_i(z_1).
    \end{aligned}
    \]
    As the product of two odd functions is odd then
    \[
    \begin{aligned}
    \int_{\reals}\phi(u_1) \kappa_1(z_1)d\gamma(z_1) = \frac{1}{\sqrt{2 \pi}} \sum_{i}^{T} \left(\phi'(d_i^+) - \phi'(d_i^-)  \right) \int_{\reals} \phi(u_1) \beta_i(z_1) d \gamma(z_1)
     = 0
    \end{aligned}
    \]
    and so 
    \begin{equation} \label{eq:righthand_H_diff}
        \int_{\reals}\phi(u_1)\left(\int_{\reals} z_2 \phi'(u_2)d\gamma(z_2) \right) d\gamma(z_1) = \sqrt{q^*(1 - \rho^2)}\int_{\reals^2}\phi(u_1) \phi''(u_2) d \gamma^{(2)}(z_1,z_2).
    \end{equation}
    Now we turn our attention to the first inner integral on the final line of \eqref{eq:H_diff_intermediate_result}. With $\rho\in (-1,0) \cup (0,1)$ and $z_1\in \reals$ fixed, and defining for typographical ease $\varphi:\reals \rightarrow \reals$ as the function $z_1 \mapsto (\phi' \circ u_2)(z_1,z_2,\rho)$, then $\phi(u_1)\phi'(u_2)$ is continuously differentiable at all points other than the following,
    \[
    \begin{aligned}
    &o_i(z_2) \defeq \frac{d_i}{\rho \sqrt{q^*}} - z_2\frac{\sqrt{1 - \rho^2}}{\rho}\\
    &e_i(z_2) \defeq -\frac{d_i}{\rho \sqrt{q^*}} - z_2\frac{\sqrt{1 - \rho^2}}{\rho}\\
    & s_i \defeq \frac{d_i}{\sqrt{q^*}}\\
    & l_i \defeq -\frac{d_i}{\sqrt{q^*}}
    \end{aligned}
    \]
    for all $i \in [T]$. Applying Lemma \ref{lemma:piecewise_integration_by_parts},
    \[
    \begin{aligned}
    \int_{\reals}z_1\phi(\sqrt{q^*}z_1)\varphi(z_1)d\gamma(z_1) &=
     - \sum_{i}^{T} \left( \left[\frac{\exp(-\frac{1}{2}z_1^2)}{\sqrt{2 \pi}} \phi(\sqrt{q^*}z_1)\varphi(z_1) \right]_{e_i^+}^{e_{i}^-} + \left[\frac{\exp(-\frac{1}{2}z_1^2)}{\sqrt{2 \pi}} \phi(\sqrt{q^*}z_1) \varphi(z_1) \right]_{o_i^+}^{o_{i}^-}\right)\\
    &- \sum_{i}^{T} \left( \left[\frac{\exp(-\frac{1}{2}z_1^2)}{\sqrt{2 \pi}} \phi(\sqrt{q^*}z_1)\varphi(z_1) \right]_{l_i^+}^{l_{i}^-} + \left[\frac{\exp(-\frac{1}{2}z_1^2)}{\sqrt{2 \pi}} \phi(\sqrt{q^*}z_1) \varphi(z_1) \right]_{s_i^+}^{s_{i}^-}\right)\\
    & + \int_{\reals}\phi(\sqrt{q^*}z_1) \varphi'(z_1) + \sqrt{q^*} \phi'(\sqrt{q^*}z_1)\varphi(z_1) d \gamma(z_1).
    \end{aligned}
    \]
    As both $\phi$ and $\varphi$ are continuous at $l_i$ and $s_i$ for all $i \in [T]$, then the left and right limits of $\phi(\sqrt{q^*}z_1)\varphi(z_1)$ at these points are equal. Therefore
    \[
    \begin{aligned}
    \int_{\reals}z_1\phi(\sqrt{q^*}z_1)\varphi(z_1)d\gamma(z_1) &=
     - \sum_{i}^{T} \left( \left[\frac{\exp(-\frac{1}{2}z_1^2)}{\sqrt{2 \pi}} \phi(\sqrt{q^*}z_1)\varphi(z_1) \right]_{e_i^+}^{e_{i}^-} + \left[\frac{\exp(-\frac{1}{2}z_1^2)}{\sqrt{2 \pi}} \phi(\sqrt{q^*}z_1) \varphi(z_1) \right]_{o_i^+}^{o_{i}^-}\right)\\
    & + \int_{\reals}\phi(\sqrt{q^*}z_1) \varphi'(z_1) + \sqrt{q^*} \phi'(\sqrt{q^*}z_1)\varphi(z_1) d \gamma(z_1)\\
    & \eqdef - \kappa_2(z_2) + \int_{\reals}\phi(\sqrt{q^*}z_1) \varphi'(z_1) + \sqrt{q^*} \phi'(\sqrt{q^*}z_1)\varphi(z_1) d \gamma(z_1).
    \end{aligned}
    \]
    Analysing $\kappa_2$, then by construction $\varphi(e_i) = \phi'(-d_i)$ and $\varphi(o_i) = \phi'(d_i)$. Similar to before it follows that
    \[
    \begin{aligned}
    \kappa_2(z_2) &= \sum_{i=1}^{T} \left( \left[\frac{\exp(-\frac{1}{2}z_1^2)}{\sqrt{2 \pi}} \phi(\sqrt{q^*}z_1)\varphi(z_1) \right]_{e_i^+}^{e_{i}^-} + \left[\frac{\exp(-\frac{1}{2}z_1^2)}{\sqrt{2 \pi}} \phi(\sqrt{q^*}z_1) \varphi(z_1) \right]_{o_i^+}^{o_{i}^-}\right)\\
    &=\sum_{i=1}^{T} \left(\frac{\exp(-\frac{1}{2}e_i^2)}{\sqrt{2 \pi}} \phi(\sqrt{q^*}e_i) \left(\varphi(-d_i^-) - \varphi(-d_i^+)  \right) + \frac{\exp(-\frac{1}{2}o_i^2)}{\sqrt{2 \pi}} \phi(\sqrt{q^*}o_i) \left(\varphi(d_i^-) - \varphi(d_i^+)  \right)   \right)\\
    &=\sum_{i=1}^{T} \left(\frac{\exp(-\frac{1}{2}e_i^2)}{\sqrt{2 \pi}} \phi(\sqrt{q^*}e_i) \left(\varphi(d_i^+) - \varphi(d_i^-)  \right) + \frac{\exp(-\frac{1}{2}o_i^2)}{\sqrt{2 \pi}} \phi(\sqrt{q^*}o_i) \left(\varphi(d_i^-) - \varphi(d_i^+)  \right)   \right)\\
    & = \frac{1}{\sqrt{2 \pi}} \sum_{i=1}^{T} \left(\left(\varphi(d_i^+) - \varphi(d_i^-)  \right) \left( \exp(-\frac{1}{2}e_i^2) \phi(\sqrt{q^*}e_i) - \exp(-\frac{1}{2}o_i^2) \phi(\sqrt{q^*}o_i)\right)  \right)
    \end{aligned}
    \]
    Note that $o_i(-z_2) = - e_i(z_2)$, therefore $\phi(\sqrt{q^*}o_i(-z_2)) = \phi(-\sqrt{q^*} e_i(z_2)) = - \phi( \sqrt{q^*}e_i(z_2))$ as $\phi$ is odd. Likewise $e_i(-z_2) = - o_i(z_2)$ and so $\phi(\sqrt{q^*}e_i(-z_2)) = \phi(-\sqrt{q^*}o_i(z_2)) = -\phi(\sqrt{q^*}o_i(z_2))$. Furthermore as
    \[
    \begin{aligned}
    o_i(z_2)^2 &= \frac{d_i^2}{\rho^2 q^*} - z_2 \frac{\sqrt{1 - \rho^2}}{\rho} + z_2^2 \frac{1 - \rho^2}{\rho^2},\\
    e_i(z_2)^2 &= \frac{d_i^2}{\rho^2 q^*} + z_2 \frac{\sqrt{1 - \rho^2}}{\rho} + z_2^2 \frac{1 - \rho^2}{\rho^2}
    \end{aligned}
    \]
    then $o_i(-z_2)^2 = e_i(z_2)^2$. Analogous to $\beta_i$, we now define $\Gamma_i(z_2) \defeq \exp(-\frac{1}{2}e_i^2) \phi(\sqrt{q^*}e_i) - \exp(-\frac{1}{2}o_i^2) \phi(\sqrt{q^*}o_i)$. The fact that $\Gamma_i(z_2)$ is odd follows from
    \[
    \begin{aligned}
    \Gamma_i(-z_2) &= \exp(-\frac{1}{2}e_i(-z_2)^2) \phi(\sqrt{q^*}e_i(-z_2))) - \exp(-\frac{1}{2}o_i(-z_2)^2) \phi(\sqrt{q^*}o_i(-z_2))\\
    & = -\exp(-\frac{1}{2}o_i(z_2)^2) \phi(\sqrt{q^*}o_i(z_2))) + \exp(-\frac{1}{2}e_i(z_2)^2) \phi(\sqrt{q^*}e_i(z_2))\\
    & = - \Gamma_i(z_2).
    \end{aligned}
    \]
    Analysing the first integral on the last line of \ref{eq:H_diff_intermediate_result}), then as $\Gamma_i$ is odd for each $i \in [T]$
    \begin{equation} \label{eq:lefthand_H_diff}
    \begin{aligned}
    &\int_{\reals} \left( \int_{\reals}  z_1\phi(u_1)\phi'(u_2)d\gamma(z_1)\right) d\gamma(z_2)\\
    &= - \int_{\reals}\kappa_2(z_2) d\gamma(z_2) + \int_{\reals}\int_{\reals}\phi(\sqrt{q^*}z_1) \varphi'(z_1) + \sqrt{q^*} \phi'(\sqrt{q^*}z_1)\varphi(z_1) d \gamma(z_1) d\gamma(z_2)\\
    & = -\frac{1}{\sqrt{2 \pi}} \sum_{i=1}^{T} \left(\varphi(d_i^+) - \varphi(d_i^-)  \right) \int_{\reals} \Gamma_i(z_2) d \gamma(z_2) + \int_{\reals^2} \phi(\sqrt{q^*}z_1) \varphi'(z_1) d \gamma^{(2)}(z_1, z_2)\\
    & + \sqrt{q^*}\int_{\reals^2} \phi'(\sqrt{q^*}z_1)\varphi(z_1) d \gamma^{(2)}(z_1, z_2)\\
    & = \sqrt{q^*}\rho \int_{\reals^2} \phi(u_1) \phi''(u_2) d \gamma^{(2)}(z_1, z_2) + \sqrt{q^*}\int_{\reals^2} \phi'(u_1)\phi'(u_2) d \gamma^{(2)}(z_1, z_2).
    \end{aligned}
    \end{equation}
    Substituting \eqref{eq:lefthand_H_diff} and \eqref{eq:righthand_H_diff} into \eqref{eq:H_diff_intermediate_result} it follows that
    \[
    \begin{aligned}
         H'(\rho)
         & = \sqrt{q^*}\left( \int_{\reals} \left( \int_{\reals} z_1\phi(u_1)\phi'(u_2)d\gamma(z_1)\right) d\gamma(z_2)  - \frac{\rho}{\sqrt{1 - \rho^2}} \int_{\reals}\phi(u_1)\left(\int_{\reals} z_2 \phi'(u_2)d\gamma(z_2) \right) d\gamma(z_1) \right)\\
         & = q^* \left(\rho \int_{\reals^2} \phi(u_1) \phi''(u_2) d \gamma^{(2)}(z_1, z_2) + \int_{\reals^2} \phi'(u_1)\phi'(u_2) d \gamma(z_1, z_2) - \rho \int_{\reals^2} \phi(u_1) \phi''(u_2) d \gamma^{(2)}(z_1, z_2) \right)\\
         & = q^* \int_{\reals^2} \phi'(u_1)\phi'(u_2) d \gamma^{(2)}(z_1, z_2).
    \end{aligned}
    \]
    It therefore follows for all $\rho \in (-1,1)$ that
    \[
    \begin{aligned}
    R_{\phi, q^*}'(\rho) = \sigma_w^2\int_{\reals^2} \phi'(u_1)\phi'(u_2) d \gamma^{(2)}(z_1, z_2)
    &=\sigma_w^2 \expec[\phi'(U_1)\phi'(U_2)]
    \end{aligned}
    \]
    as claimed.
\end{proof}

\subsection{Lemma \ref{lemma:bound_from_identity}} \label{appendix:proof_bound_from_identity}
\textbf{Lemma \ref{lemma:bound_from_identity}.} \textit{Under the same conditions and assumptions as in Lemma \ref{lemma:omega_corr_diff}, it holds that
	\[
    \max_{\rho \in [0,1]}|R_{\phi, q^*}(\rho) - \rho| = \frac{\sigma_b^2}{q^*}.
    \]}

\begin{proof}
Observe that
\[
	\begin{aligned}
	R_{\phi, q^*}(0) &= \frac{\expec[\phi(\sqrt{q^*}Z_1) \phi(\sqrt{q^*}Z_2)]}{q^*\expec[\phi^{'}(\sqrt{q^*}Z)^2]} + \frac{\sigma_b^2}{q^*}\\
	& = \frac{\expec[\phi(\sqrt{q^*}Z_1)]^2}{q^*\expec[\phi^{'}(\sqrt{q^*}Z)^2]} + \frac{\sigma_b^2}{q^*}.
	\end{aligned}
\]
As $\phi$ is odd then $\expec[\phi(\sqrt{q^*}Z_1)] = 0$ and therefore
\[
R_{\phi, q^*}(0) = \frac{\sigma_b^2}{q^*}>0.
\]
Lemma \ref{lemma:bound_from_identity} follows then as long as $|R_{\phi, q^*}(\rho) - \rho|< R_{\phi, q^*}(0)$ for all $\rho \in [0,1]$. As
\[
\begin{aligned}
&R_{\phi, q^*}(1) = \frac{\sigma_w^2}{q^*}\expec[\phi(\sqrt{q^*}Z_1)^2] + \frac{\sigma_b^2}{q^*} = \frac{V(q^*)}{q^*} = 1
\end{aligned}
\]
then $|R_{\phi, q^*}(1) - 1| = 0 < R_{\phi, q^*}(0)$. All that remains to show is that the inequality holds for $\rho \in (0,1)$. We proceed using an approach similar to that used to prove Proposition 3 in \cite{hayou19a}. Using Lemma \ref{lemma:omega_corr_diff} then for any  $\rho \in (0,1)$ we have
	\[
	\begin{aligned}
	R_{\phi, q^*}'(\rho) &= \frac{\expec[\phi'(U_1) \phi'(U_2)]}{\expec[\phi^{'}(\sqrt{q^*}Z)^2]}\\
	&\leq \frac{\expec[\phi'(\sqrt{q^*}Z_1)^2]^{\frac{1}{2}} \expec[\phi'(U_2)^2]^{\frac{1}{2}} }{\expec[\phi^{'}(\sqrt{q^*}Z)^2]}\\
	& = \frac{\expec[\phi'(U_2)^2]^{\frac{1}{2}} }{\expec[\phi^{'}(\sqrt{q^*}Z)^2]^{\frac{1}{2}}}\\
	& =1.
	\end{aligned}
	\]
The inequality on the second line of the above follows from Cauchy-Schwarz and the equalities on the third and fourth lines are due to the fact that $Z_1, Z, U_2 \sim \cN(0,1)$ are all identically distributed. Note that equality holds iff either $\rho=0$ or there exists an $\alpha \in \reals$ such that $\phi'(U_1) = \alpha\phi'(U_2)$. Since $Z_1$ and $Z_2$ are i.i.d. this can only occur if $\phi'$ is a constant, which in turn would imply that $\phi$ must be linear. However, by construction linear functions are clearly not scaled-bounded activations and therefore for any $\rho \in (0,1)$ it holds that
	\[
	R_{\phi, q^*}'(\rho) < 1.
	\]
    For $\rho \in (0,1)$ then integrating both sides of the above inequality and applying the fundamental theorem of calculus we have
	\[
	\int_0^{\rho} R_{\phi, q^*}'(t) dt < \int_0^{\rho}1 dt \implies 	R_{\phi, q^*}(\rho) - \rho  < R_{\phi, q^*}(0)
	\]
	and 
	\[
	\int_{\rho}^1 R_{\phi, q^*}'(t) dt < \int_{\rho}^1 1 dt  \implies  \rho - R_{\phi, q^*}(\rho)  < 0.
	\]
    As $R_{\phi, q^*}(0)>0$ then we conclude for $\rho \in (0,1)$ that $|R_{\phi, q^*}(\rho) - \rho|< R_{\phi, q^*}(0)$. Therefore
    \[
    |R_{\phi, q^*}(\rho) - \rho| \leq R_{\phi, q^*}(0) = \frac{\sigma_b^2}{q^*}
    \]
    for all $\rho \in [0,1]$ as claimed.
\end{proof}

\subsection{Lemma \ref{lemma:a_q_ratio_bound}} \label{appendix:proof_a_q_ratio_bound}
\textbf{Lemma \ref{lemma:a_q_ratio_bound}.} \textit{
Under the same conditions and assumptions as in Lemma \ref{lemma:omega_corr_diff}, and defining $y:=\frac{\sigma_b^2}{a^2}$, then
    \[
        \Lambda(y) < \frac{a}{\sqrt{q^*}} < \left(\frac{8}{\pi}\right)^{1/6} y^{-1/3},
    \]
    where $\Lambda(y)$ is defined as
    \[
    \left(W_0 \left(\frac{2}{\pi}\left(\sqrt{\frac{8}{\pi}}\exp \left(- \frac{\left(\sqrt{W_0 \left(\frac{2}{\pi}y^{-2}\right)}\right)^2}{2}\right) \left( \frac{1}{\left(\sqrt{W_0 \left(\frac{2}{\pi}y^{-2}\right)}\right)} + \left(\sqrt{W_0 \left(\frac{2}{\pi}y^{-2}\right)}\right)\right) \right)^{-2}\right)\right)^{1/2}
    \]
    and $W_0$ denotes the principal branch of the Lambert W function.
}

\begin{proof}
To derive the upper bound on $a/\sqrt{q^*}$ we study lower bounds for $V_{\phi}(q)$. To this end we first lower bound $\expec[\phi(\sqrt{q}Z)^2]$. The fact that $\phi$ is odd implies $\phi^2$ is even, therefore
    \[
	\begin{aligned}
		\expec[\phi(\sqrt{q}Z)^2] &=  2\left(\int_{0}^{a/\sqrt{q}} k^2q z^2 d\gamma(z) +  \int_{a/\sqrt{q}}^{\infty} \phi(\sqrt{q}z)^2  d\gamma(z) \right) \\
		& \geq 2\left(k^2q \int_{0}^{a/\sqrt{q}} z^2  d\gamma(z) \right)\\
		& = qk^2 \text{erf}\left( \frac{a}{\sqrt{2q}}\right) - \sqrt{\frac{2}{\pi}}ak^2 \sqrt{q} \exp\left(-\frac{a^2}{2q}\right).
	\end{aligned}
	\]
To now upper bound $\expec[\phi^{'}(\sqrt{q*}Z)^2]$ we use the fact that $|\phi'(z)| \leq k$, which implies
	\[
	\begin{aligned}
		\expec[\phi^{'}(\sqrt{q^*}Z)^2] &\leq  2k^2 \int_{0}^{\infty} d \gamma(z)\\
		& = k^2.
	\end{aligned}
	\]
As $q^* = V_{\phi}(q^*)$, it therefore follows that
\[
q^* \geq q^* \text{erf}\left( \frac{a}{\sqrt{2q^*}}\right) - \sqrt{\frac{2}{\pi}}a \sqrt{q^*} \exp\left(-\frac{a^2}{2q^*}\right) + \sigma_b^2.
\]
Rearranging and dividing by $a^2$ gives
\[
\frac{q^*}{a^2} \text{erfc}\left( \frac{a}{\sqrt{2q^*}}\right) \geq \frac{\sigma_b^2}{a^2} -  \frac{\sqrt{q^*}}{a} \sqrt{\frac{2}{\pi}} \exp\left(-\frac{a^2}{2q^*}\right).
\]
For typographical ease we now substitute $x = \frac{a}{\sqrt{q^*}}$ and multiply by $x^2$,
\[
\text{erfc}\left( \frac{x}{\sqrt{2}}\right) \geq x^2\frac{\sigma_b^2}{a^2} - x \sqrt{\frac{2}{\pi}} \exp\left(-\frac{x^2}{2}\right).
\]
It is known that $\text{erfc}\left( \frac{x}{\sqrt{2}}\right) \leq \sqrt{\frac{2}{\pi}} \frac{1}{x} \exp\left(-\frac{x^2}{2}\right) \leq \sqrt{\frac{2}{\pi}} \frac{1}{x}$ (see e.g., \cite{4289989}). Additionally it holds that $\exp\left(-\frac{x^2}{2}\right) < \frac{1}{x^2}$, this follows from the fact that $x> 2\ln(x)$, which can be proved using elementary calculus. Using these results we formulate the following inequality,
\[
x^2\frac{\sigma_b^2}{a^2} - x \sqrt{\frac{2}{\pi}} \frac{1}{x^2} < \sqrt{\frac{2}{\pi}}\frac{1}{x} ,
\]
which simplifies to
\[
\frac{a}{\sqrt{q^*}} \eqdef x < \left(\frac{8}{\pi}\right)^{1/6} \frac{a^{2/3}}{\sigma_b^{2/3}} = \left(\frac{8}{\pi}\right)^{1/6} y^{-1/3}.
\]
as claimed.

The derivation of the lower bound is more involved, we proceed by deriving an upper bound on $V_{\phi}(q)$ which will allow us to upper bound $q^*$. First, as $\phi$ is upper bounded by $k^2a^2$ then
\[
	\begin{aligned}
		\expec[\phi(\sqrt{q}Z)^2] & \leq  2\left(\int_{0}^{a/\sqrt{q}} k^2q z^2 d \gamma(z) +  k^2a^2\int_{a/\sqrt{q}}^{\infty} d\gamma(z) \right) \\
		&= k^2q \text{erf}\left( \frac{a}{\sqrt{2q}}\right) - k^2 \sqrt{\frac{2}{\pi}} a \sqrt{q} \exp \left(- \frac{a^2}{2q} \right) +  k^2a^2\text{erfc}\left( \frac{a}{\sqrt{2q}}\right).
	\end{aligned}
	\]
A simple lower bound for $\expec[\phi^{'}(\sqrt{q^*}Z)^2]$ is as follows,
	\[
	\begin{aligned}
		\expec[\phi^{'}(\sqrt{q^*}Z)^2] & \geq  2k^2 \int_{0}^{a/\sqrt{q^*}} \gamma(z)\\
		& = k^2 \text{erf}\left( \frac{a}{\sqrt{2q^*}}\right).
		\end{aligned}
	\]
As a result of these inequalities we may formulate the following inequality,
\[
q^* \leq \frac{k^2q^* \text{erf}\left( \frac{a}{\sqrt{2q^*}}\right) - k^2 \sqrt{\frac{2}{\pi}} a \sqrt{q^*} \exp \left(- \frac{a^2}{q^*} \right) +  k^2a^2\text{erfc}\left( \frac{a}{\sqrt{2q^*}}\right)}{k^2 \text{erf}\left( \frac{a}{\sqrt{2q^*}}\right)}  + \sigma_b^2.
\]
Rearranging this expression gives
\[
q^*\text{erf}\left( \frac{a}{\sqrt{2q^*}}\right) \leq q^* \text{erf}\left( \frac{a}{\sqrt{2q^*}}\right) -  \sqrt{\frac{2}{\pi}} a \sqrt{q^*} \exp \left(- \frac{a^2}{2q^*} \right) +  a^2\text{erfc}\left( \frac{a}{\sqrt{2q^*}}\right) + \sigma_b^2\text{erf}\left( \frac{a}{\sqrt{2q^*}}\right).
\]
Further simplification leads to
\[
\begin{aligned}
\sqrt{\frac{2}{\pi}} a \sqrt{q^*} \exp \left(- \frac{a^2}{2q^*} \right) &\leq (a^2 - \sigma_b^2)\text{erfc}\left( \frac{a}{\sqrt{2q^*}}\right) + \sigma_b^2\\
& <  a^2\text{erfc}\left( \frac{a}{\sqrt{2q^*}}\right) + \sigma_b^2.
\end{aligned}
\]
Dividing by $a^2$ and making the substitution $x = \frac{a}{\sqrt{q^*}}$ we arrive at the key inequality
\begin{equation}\label{eq:key_ratio_proof_inequality}
\sqrt{\frac{2}{\pi}} \frac{1}{x}\exp \left(- \frac{x^2}{2}\right)  <  \text{erfc}\left( \frac{x}{\sqrt{2}}\right) + \frac{\sigma_b^2}{a^2}.
\end{equation}
Recalling that our objective is to lower bound the quantity $\frac{a}{\sqrt{q^*}}$, then by construction any value of $x$ which satisfies the above inequality is a viable candidate. The challenging aspect now is to find a non-trivial candidate, in particular one that scales appropriately with $a$ and $\sigma_b^2$. For typographical ease we now define $g(x):= \sqrt{\frac{2}{\pi}} \frac{1}{x}\exp \left(- \frac{x^2}{2}\right)$ and $h(x): = \text{erfc}\left( \frac{x}{\sqrt{2}}\right) + \frac{\sigma_b^2}{a^2}$. The derivatives of each of these function are as follows,
\[
\begin{aligned}
g'(x) &=  -\sqrt{\frac{2}{\pi}}\exp \left(- \frac{x^2}{2}\right) \left(\frac{1}{x^2} + 1\right) \\
h'(x) &= -\sqrt{\frac{2}{\pi}}\exp \left(- \frac{x^2}{2}\right),
\end{aligned}
\]
note that $g'(x) > h'(x)$ for all $x\in \reals$. We now identify a candidate lower bound by considering the linear equation $u(x)$ defined by being tangent to $g(x)$ at $x = \delta:= \sqrt{W_0 \left(\frac{2}{\pi}y^{-2}\right)}$. Here $W_0$ refers to the principle branch of the Lambert W function (see e.g., \cite{lambert}). The line $u(x)$ therefore has the form
\[
u(x) =  -\sqrt{\frac{2}{\pi}}\exp \left(- \frac{\delta^2}{2}\right) \left(\frac{1}{\delta^2} + 1\right)x + \beta.
\]
Observe that as $u(\delta) = y = g(\delta)$ by construction, then
\[
\begin{aligned}
\beta &= y + \sqrt{\frac{2}{\pi}}\exp \left(- \frac{\delta^2}{2}\right)\left( \frac{1}{\delta} +  \delta\right)\\
&= \sqrt{\frac{2}{\pi}} \frac{1}{\delta}\exp \left(- \frac{\delta^2}{2}\right) + \sqrt{\frac{2}{\pi}}\exp \left(- \frac{\delta^2}{2}\right)\left( \frac{1}{\delta} +  \delta\right)\\
& =  \sqrt{\frac{2}{\pi}}\exp \left(- \frac{\delta^2}{2}\right) \left( \frac{2}{\delta} + \delta\right)
\end{aligned}
\]
We now identify a candidate lower bound $\gamma$ as the solution of $g(\gamma) = \beta$, the solution of which can be expressed using the Lambert W function, $\gamma = \sqrt{W_0 \left(\frac{2}{\pi}\beta^{-2}\right)}$. Consider the following compositions of functions, $\Phi_g(y) := g(\gamma(\beta(\delta(y))))$ and $\Psi_h(y) := h(\gamma(\beta(\delta(y))))$. To ensure that $\gamma \leq \frac{a}{\sqrt{q^*}}$, i.e., that $\gamma$ is indeed a valid lower bound, then it must hold that $\Phi_g(y) \geq \Phi_h(y)$. We refer the reader to Figure \ref{fig:proof_illustration} for a graphical description of this approach to identifying a candidate lower bound. Our task then is to inspect the range of $y \in \reals_{\geq 0}$ values for which this inequality holds true. First we note that $\Phi_f(0)= \Phi_h(0) = 0$. Second we observe $h, g, \gamma, \beta$ and $\delta$ are differentiable functions on $\reals_{> 0}$. As
\[
\begin{aligned}
\frac{d \delta}{dy} &= - \frac{ \sqrt{W_0 \left(\frac{2}{\pi}y^{-2}\right)}}{y\left(W_0 \left(\frac{2}{\pi}y^{-2}\right) + 1 \right)} < 0,\\
\frac{d \beta}{d \delta} &= - \frac{\sqrt{\frac{2}{\pi}} \exp \left(- \frac{\delta^2}{2}\right) \left(  \delta^4 + \delta^2 + 2 \right)}{\delta^2} < 0,\\
\frac{d \gamma}{d \beta} &= - \frac{ \sqrt{W_0 \left(\frac{2}{\pi}\beta^{-2}\right)}}{\beta\left(W_0 \left(\frac{2}{\pi}\beta^{-2}\right) + 1 \right)} < 0
\end{aligned}
\]
then $\frac{d \gamma}{dy} =\frac{d \gamma}{d \beta} \frac{d \beta}{d \delta}\frac{d \delta}{dy} <0$. Applying the chain rule
\[
\begin{aligned}
\Phi_g'(y) &= \frac{dg}{d\gamma} \frac{d \gamma}{dy}\\
&= -\sqrt{\frac{2}{\pi}}\exp \left(- \frac{\gamma^2}{2}\right) \left(\frac{1}{\gamma^2} + 1\right) \frac{d \gamma}{dy}\\
& = \sqrt{\frac{2}{\pi}}\exp \left(- \frac{\gamma^2}{2}\right) \left(\frac{1}{\gamma^2} + 1\right) \left| \frac{d \gamma}{dy} \right|
\end{aligned}
\]
and
\[
\begin{aligned}
\Phi_h'(y)&= \frac{dh}{d\gamma} \frac{d \gamma}{dy}\\
&= \sqrt{\frac{2}{\pi}}\exp \left(- \frac{\gamma^2}{2}\right) \left| \frac{d \gamma}{dy} \right|.
\end{aligned}
\]
The fact that
\[
\begin{aligned}
\sqrt{\frac{2}{\pi}}\exp \left(- \frac{\gamma^2}{2}\right) \left(\frac{1}{\gamma^2} + 2\right) \left| \frac{d \gamma}{dy} \right| >  \sqrt{\frac{2}{\pi}}\exp \left(- \frac{\gamma^2}{2}\right) \left| \frac{d \gamma}{dy} \right|,
\end{aligned}
\]
which can be further simplified to
\[
\left(\frac{1}{\gamma^2} + 1\right) > 1,
\]
holds for all $\gamma \in \reals$ implies that $\Phi_g'(y)>\Phi_h'(y)$ for all $y \in \reals_{>0}$. Therefore
\[
\begin{aligned}
\int_{0}^{y}\Phi_g'(t) - \Phi_h'(t) dt & > 0
\end{aligned}
\]
for all $y \in \reals_{>0}$. Now applying the fundamental theorem of calculus
\[
\begin{aligned}
\int_{0}^{y}\Phi_g'(t)  - \Phi_h'(t) dt &=  \Phi_g(y) -  \Phi_h(y) + \Phi_h(0) - \Phi_f(0)\\
& =  \Phi_g(y) -  \Phi_h(y)
\end{aligned}
\]
we conclude that $\Phi_g(y) > \Phi_h(y)$ for all $y \in \reals_{> 0}$. As a result, $\gamma$ is a valid lower bound for $\frac{a^2}{\sqrt{q^*}}$ as long as $\frac{\sigma_b^2}{a^2}>0$. Finally, to recover the statement of the theorem, we define the composite function $\Lambda: \reals_{\geq0} \rightarrow \reals_{\geq0}$ as $\Lambda(y) \defeq \gamma( \beta ( \delta(y)))$.
\end{proof}

\begin{figure}[htbp]
\centering
\includegraphics[scale=0.6]{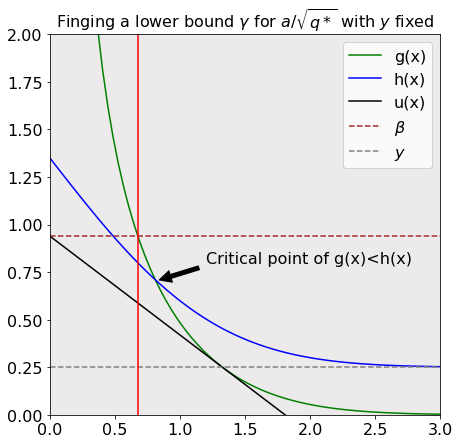}
\caption{illustration of proof approach to finding a lower bound $\gamma$ for $\frac{a}{\sqrt{q^*}}$ for a given value of $y = \sigma_b^2/a^2$ (in this example $y=0.25$). The red line indicates the value on the x-axis defined to be $\gamma$. This is computed by first identifying where the line $u(x)$, which lies tangent to the point where $g$ hits the asymptotic limit of $h(x)$ which in turn corresponds to the value of $y$, intercepts the y-axis, denoted as $\beta$. We define $\gamma$ then to be the point at which the horizontal line at $\beta$ (the upper dashed brown line) intercepts $g(x)$. To ensure that $\beta$ intercepts $g(x)$ above the critical point of the inequality $g(x) < h(x)$, it is necessary to check that $h(\gamma) < g(\gamma)$ is true.}
\end{figure}\label{fig:proof_illustration}

\newpage
\subsection{Convergence of preactivation correlations with depth for a range of activation functions} \label{appendix:conv_of_corrs}

\begin{figure}[H] 
\centering
\subfloat{ \includegraphics[scale=.45]{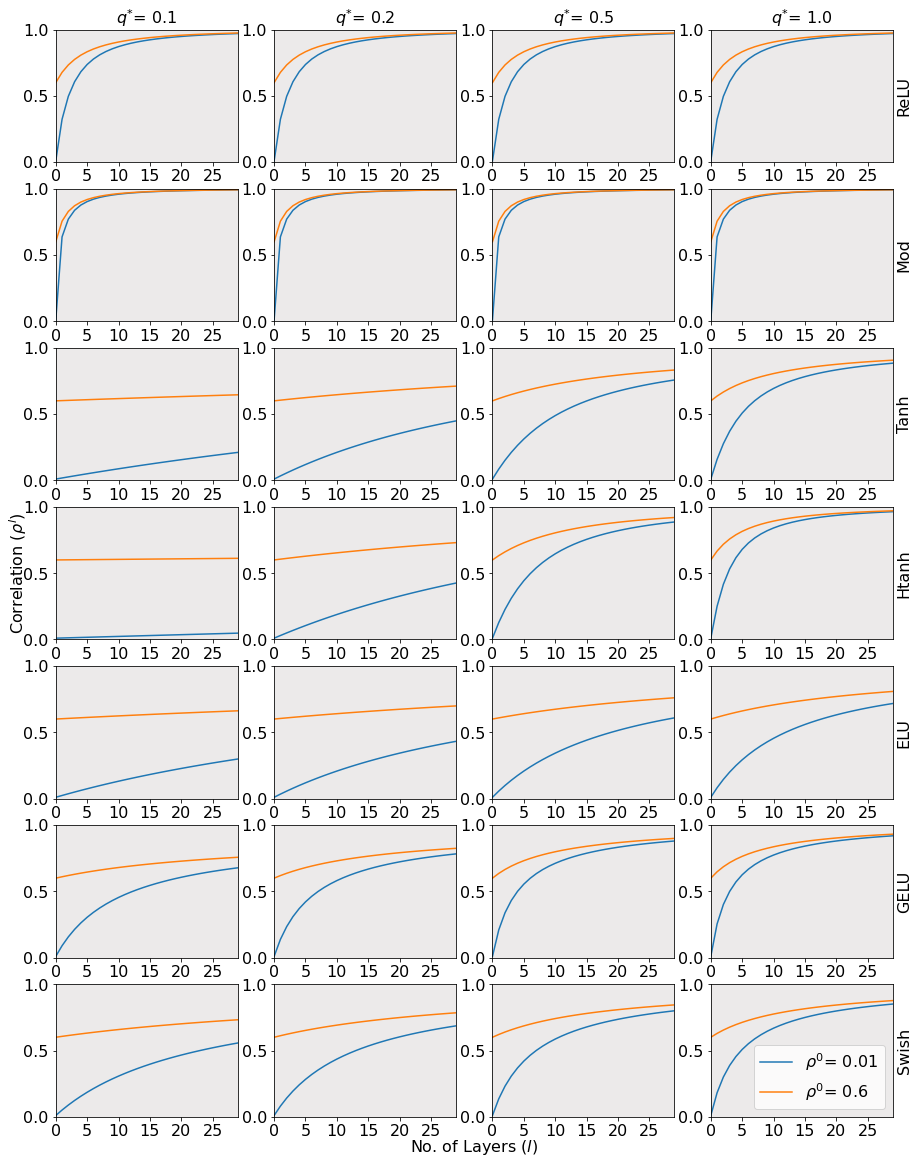}}
 \caption{evolution of numerically computed preactivation correlations from two initial correlations, 0.01 and 0.6, across a range of activation functions and fixed point values.} \label{fig:conv_of_corrs}
\end{figure}

\end{document}